\documentclass{article}

  \usepackage[preprint, nonatbib]{neurips_2023}

\usepackage[utf8]{inputenc} % allow utf-8 input
\usepackage[T1]{fontenc}    % use 8-bit T1 fonts
\usepackage{hyperref}       % hyperlinks
\usepackage{url}            % simple URL typesetting
\usepackage{booktabs}       % professional-quality tables
\usepackage{amsfonts}       % blackboard math symbols
\usepackage{nicefrac}       % compact symbols for 1/2, etc.
\usepackage{microtype}      % microtypography
\usepackage{xcolor}         % colors
\usepackage{graphicx}
\usepackage{amsmath}
\usepackage{amsthm}
\usepackage{enumitem}
\usepackage{apptools}
\usepackage{caption}
\usepackage{subcaption}

\ifx\assumption\undefined
\newtheorem{assumption}{Assumption}
\fi

\newtheorem{theorem}{Theorem}
\AtAppendix{\counterwithin{lemma}{section}}
\newtheorem{lemma}{Lemma}

\title{Multiply Robust Federated Estimation of Targeted Average Treatment Effects}

\author{%
  Larry Han\\
  Department of Health Sciences\\
  Northeastern University\\
  Boston, MA 02115 \\
  \texttt{lar.han@northeastern.edu} \\
  \And
  Zhu Shen\\
  Department of Biostatistics\\
  Harvard University\\
  Boston, MA 02115 \\
  \texttt{zhushen@g.harvard.edu} \\
  \AND
  Jose Zubizarreta\\
  Departments of Health Care Policy, Biostatistics, and Statistics\\
  Harvard University\\
  Boston, MA 02115 \\
  \texttt{zubizarreta@hcp.med.harvard.edu} \\
}

\begin{document}

\maketitle

\begin{abstract}
Federated or multi-site studies have distinct advantages over single-site studies, including increased generalizability, the ability to study underrepresented populations, and the opportunity to study rare exposures and outcomes. However, these studies are challenging due to the need to preserve the privacy of each individual's data and the heterogeneity in their covariate distributions. We propose a novel federated approach to derive valid causal inferences for a target population using multi-site data. We adjust for covariate shift and covariate mismatch between sites by developing multiply-robust and privacy-preserving nuisance function estimation. Our methodology incorporates transfer learning to estimate ensemble weights to combine information from source sites. We show that these learned weights are efficient and optimal under different scenarios. We showcase the finite sample advantages of our approach in terms of efficiency and robustness compared to existing approaches.
\end{abstract}

\section{Introduction}
Compared to single-site studies, federated or multi-site studies confer distinct advantages, such as the potential for increased generalizability of findings, the opportunity to learn about underrepresented populations, and the ability to study rare exposures and outcomes. However, deriving valid causal inferences using multi-site data is difficult due to numerous real-world challenges, including \textit{heterogeneity of site populations}, \textit{different data structures}, and \textit{privacy-preserving constraints} stemming from policies such as the General Data Protection Regulation (GDPR) and Health Insurance Portability and Accountability Act (HIPAA) that prohibit direct data pooling.

Recent methodological developments have focused on privacy-preserving estimation strategies. These strategies typically involve sharing summary-level information from multiple data sources \cite{vo2021federated, xiong2021federated, han2021federated, han2022privacy, kawamata2022federated}. However, they often require restrictive assumptions such as homogeneous data structures and model specifications (e.g., a common set of observed covariates measured using a common data model), which are not realistic in practice. 

To address these methodological gaps, we propose a \textit{multiply robust} and \textit{privacy-preserving} estimator that leverages multi-site information to estimate causal effects in a \textit{target population of interest}. Compared to existing approaches, our method allows investigators from different sites to incorporate \textit{site-specific covariate information and domain knowledge} and provides \textit{increased protection against model misspecification}. Our method allows for flexible identification under different settings, including systematically missing covariates and different site-specific covariates (termed \textit{covariate mismatch}). Our proposed method adopts an \textit{adaptive ensembling approach that optimally combines estimates from source sites} and serves as a data-driven metric for the transportability of source sites. Moreover, the proposed method relaxes the assumption of homogeneous model specifications by adopting a class of multiply robust estimators for estimating the nuisance functions. 

\subsection{Related Work and Contributions}
The current literature on multi-site causal inference typically assumes that a common set of confounders is observed in all sites \cite{ dahabreh2019generalizing, dahabreh2020extending, han2021federated, han2022privacy}. However, this assumption is rarely met  due to variations in local practices, e.g., differing data collection standards and coding practices. In particular, the target site often lacks data on certain covariates available in the source sites, and ignoring them can result in biased and inefficient inference \cite{zeng2023efficient}. Recently, \cite{yang2020combining} proposed a method to address covariate mismatch by integrating source samples with unmeasured confounders and a target sample containing information about these confounders. However, they assume that the target sample is either a simple random sample or is sampled according to a known procedure from the source population. \cite{guo2021multi} extended the method to the setting where the average treatment effect (ATE) is not identifiable in some sites by constructing control variates. However, their approach is limited to addressing selection biases in case-control studies and cannot be easily extended to other outcome types. \cite{zeng2023efficient} extended the framework by \cite{dahabreh2019generalizing, dahabreh2020extending} to handle covariate mismatch by regressing predicted conditional outcomes on effect modifiers and then taking the means of these regression models evaluated on target site samples. Our work extends \cite{zeng2023efficient} to the multi-site and federated data setting by utilizing an adaptive weighting approach that optimally combines estimates from source sites.

Most existing approaches in the generalizability and transportability literature deal with heterogeneous covariate distributions by modeling site selection processes \cite{cole2010generalizing, buchanan2018generalizing, tipton2013improving, dahabreh2019generalizing, dahabreh2020extending}. For instance,
covariate shift can be accounted for via inverse probability of selection weights \cite{cole2010generalizing, buchanan2018generalizing}, stratification \cite{tipton2013improving}, or augmentation \cite{dahabreh2019generalizing, dahabreh2020extending}, which require pooling individual-level information across sites. Our work differs from those in that we preserve individual data privacy, sharing only summary-level information about the target site. Specifically, we adopt density ratio models \cite{qin1998inferences, duan2020learningb} that only share covariate moments of the target samples. Under certain specifications, these density ratio models are equivalent to logistic regression-based selection models for adjusting heterogeneity between target and source populations. Our approach shares similarities with calibration weighting methods, but we employ semi-parametric efficiency theory to enable a closed-form approximation of variance rather than relying on the bootstrap. 

Further, when data sources are heterogeneous, it would be beneficial for investigators at different sites to incorporate site-specific knowledge when specifying candidate models. However, to the best of our knowledge, existing methods require common models to be specified across sites, which may not be realistic or flexible enough \cite{vo2021federated, xiong2021federated, han2021federated, han2022privacy, kawamata2022federated}. We relax this requirement by adopting a multiply robust estimator, allowing investigators in each site to propose multiple, different outcome and treatment models. The estimator is consistent if any one of the multiple outcome or treatment models is correctly specified. Our work builds on \cite{li2020demystifying}, which established an equivalence between doubly robust and multiply robust estimators using mixing weights determined by predictive risks of candidate models  \cite{han2013estimation, han2014further, han2014multiply, chan2013simple, chan2014oracle, chen2017multiply}. Further, to avoid negative transfer due to non-transportable source estimates, we adopt a data-adaptive ensembling approach \cite{han2021federated, han2022privacy} that guarantees that the federated estimator achieves improved precision compared to an estimator using target site data alone when at least one source estimate is sufficiently similar to the target estimate.

\section{Preliminaries}

We consider data from $K$ sites, where each site has access to its individual-level data but is prohibited from sharing this data with any other site. The set of sites will be denoted by $\mathcal{K} = \{1, 2, ..., K\}$. Without loss of generality, we define the target site to be the first site, i.e., $T = \{1\}$ and the source sites as the remaining sites, i.e., $\mathcal{S} = \mathcal{K} \setminus T = \{2, ..., K\}$.
% Our framework is adaptable to different target population definitions, such as specific covariate profiles.

For individual $i$, let $Y_{i}$ denote an observed outcome, which can be continuous or discrete.  $X_{i} \in \mathbb{R}^{p}$ represents the $p$-dimensional baseline covariates in source site $k \in \mathcal{S}$. $V_{i} \in \mathbb{R}^{q}$ represents the (partial) baseline covariates in the target site $T$ such that $V_{i} \subseteq X_i$. To simplify the presentation, we assume an identical set of covariates across all source sites, although our method can accommodate scenarios where distinct covariate sets are present among the source sites. Let $A_{i}$ represent a binary treatment indicator, with $A_{i} =1$ denoting treatment and $A_{i} = 0$ denoting control. $R_{i}$ is a site indicator with $R_{i} = k$ if patient $i$ is from the site $k$. We observe $n_T$ target samples, $D_T = \{Y_{i}, V_{i}, A_{i}, R_{i} = T, 1 \leq i \leq n_T\}$ and $n_k$ source samples, $D_k = \{Y_{i}, X_{i}, A_{i}, R_{i} = k, 1 \leq i \leq n_k\}$ for each $k \in \mathcal{S}$. The total sample size is $N = \sum_{k \in \mathcal{K}} n_k$. Under the potential outcomes framework \cite{neyman1923application, rubin1974estimating}, we denote the counterfactual outcomes under treatment and control as $\left\{Y_{i}(1), Y_{i}(0) \right\}$, and only one of them is observed: $Y_{i} = A_{i} Y_{i}(1) + (1-A_{i})Y_{i}(0)$ \cite{rubin1980randomization}. The data structure is illustrated in Figure \ref{figure_non_nested}.

\begin{figure}[!htbp]
  \centering
  \includegraphics[width=0.7\textwidth]{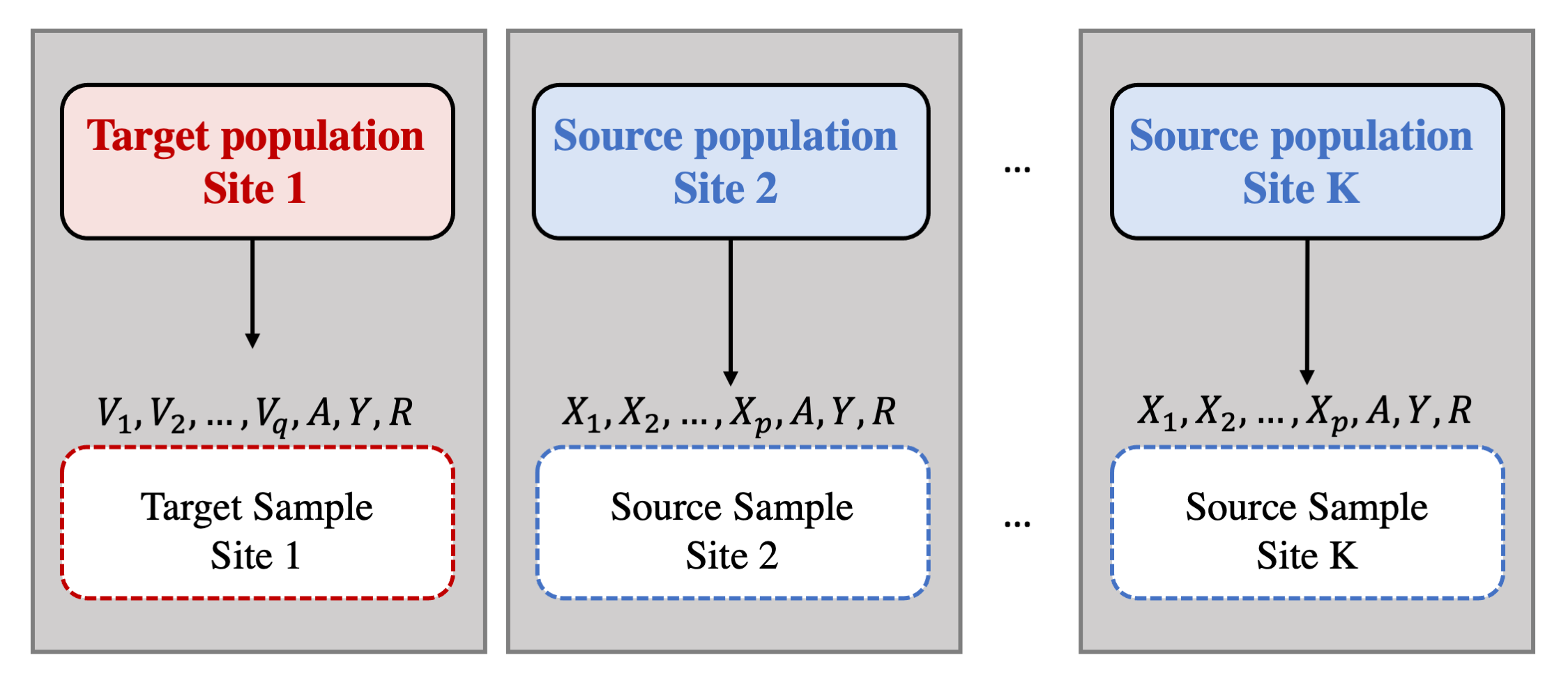}
  \caption{Schematic of the data structure in the multi-site setting.}
  \label{figure_non_nested}
\end{figure}

Our goal is to estimate the average treatment effect in the target population,
\begin{equation}
    \label{equation_TATE}
    \Delta_T = \mu_{1, T}-\mu_{0, T} \quad \text{ where } \quad \mu_{a, T} = {E}\left\{Y_i{(a)} \mid  R_i = T \right\} \text{ for } a \in \{0, 1\},
\end{equation}

where $\mu_{a, T}$ is the mean potential outcome under treatment $a$ in the target population. To identify this quantity, we consider the following assumptions: 

\begin{enumerate}[label=({A\arabic*})]
    \item \label{assumption_consistency} (Consistency): For every individual $i$, if $A_i = a$, then $Y_i = Y_i{(a)}$. 
    \item \label{assumption_target_unconfoundedness}
    (Mean exchangeability over treatment assignment in the target population): $ \\ {E}\{Y_i{(a)} \mid {V}_{i} = v, A_i, R_i = T \} = {E}\left\{ Y_i{(a)} \mid {V}_{i} = v, R_i = T \right\}$.
    \item \label{assumption_target_positivity}
    (Positivity of treatment assignment in the target population): $\\ 0 < P(A_i = 1 \mid {V}_{i} = v, R_i = T) < 1$ for any $v$ s.t. $P( {V}_{i} = v \mid R_i = T) > 0$.
    \item \label{assumption_source_unconfoundedness}
    (Mean exchangeability over treatment assignment in the source populations): $\\ {E}\{Y_i{(a)} \mid {X}_{i} = x, A_i, R_i = k \} = {E}\left\{ Y_i{(a)} \mid {X}_{i} = x, R_i = k \right\}$, $k \in \mathcal{S}$.  \item \label{assumption_source_positivity}(Positivity of treatment assignment in the source populations): $\\ 0 < P(A_i = 1 \mid {X}_{i} = x, R_i = k) < 1$ for any $x$ s.t. $P( {X}_{i} = x \mid R_i = k) > 0$, $k \in \mathcal{S}$.  
    \item \label{assumption_selection_unconfoundedness}(Mean exchangeability over site selection): $\\ {E}\{Y_i{(a)} \mid V_{i} = v, R_i = k\} = {E}\left\{ Y_i{(a)} \mid V_{i} = v\right\}$, $k \in \mathcal{K}$. 
    \item \label{assumption_selection_positivity}(Positivity of site selection): $\\ 0 < P(R_i = k \mid {V}_{i} = v) < 1$ for $k \in \mathcal{S}$ and any $v$ s.t. $P({V}_{i} = v) > 0$. 
\end{enumerate}

Assumption \ref{assumption_consistency} is the stable unit treatment value assumption (SUTVA), requiring no interference between individuals. Assumption \ref{assumption_target_unconfoundedness} (Assumption \ref{assumption_source_unconfoundedness}) states that the mean counterfactual outcome under treatment $a$ is independent of treatment assignment, conditional on baseline covariates in the target (source) populations. For Assumption \ref{assumption_target_unconfoundedness} and \ref{assumption_source_unconfoundedness} to hold, we require all effect modifiers to be measured in $V$. Assumption \ref{assumption_target_positivity} (Assumption \ref{assumption_source_positivity}) states that each individual in the target (source) populations has a positive probability of receiving each treatment. Assumption \ref{assumption_selection_unconfoundedness} states that the mean counterfactual outcome is independent of site selection, conditional on covariates in the target population. For Assumption \ref{assumption_selection_unconfoundedness} to hold, we require all covariates that are distributed differently between target and source populations (shifted covariates) to be measured in $V$. Thus, if these effect modifiers are measured in $V$, Assumption \ref{assumption_target_unconfoundedness}, \ref{assumption_source_unconfoundedness} and \ref{assumption_selection_unconfoundedness} automatically hold. Assumption \ref{assumption_selection_positivity} requires that in each stratum defined by $V$, the probability of being in a source population for each individual is positive. Theorem \ref{theorem_identification} shows that under Assumption \ref{assumption_consistency}, \ref{assumption_source_unconfoundedness} - \ref{assumption_selection_positivity}, the mean counterfactual outcome for the target can be identified in the sources. Since these assumptions may not hold in practice, we devise a data-adaptive ensembling procedure in Section \ref{federated} to screen out sites that significantly violate these assumptions.   
\begin{theorem} 
    \label{theorem_identification}
% theorems to show identification of target only and source site for TATE. 
    If Assumptions \ref{assumption_consistency} -  \ref{assumption_target_positivity} hold, the mean counterfactual outcomes in the target population can be identified using the target sample.
    \begin{equation}
        \mu_{a, T} = E\left\{ Y_i{(a)} \mid R_i = T\right\} 
        = E\left\{ E\left\{ Y_i \mid V_{i} = v, A_i = a, R_i = T \right\} \mid R_i = T \right\}. 
        \label{equation_target_identification}
    \end{equation}
    
    If Assumptions \ref{assumption_consistency}, \ref{assumption_source_unconfoundedness} -  \ref{assumption_selection_positivity} hold, the mean counterfactual outcomes in the target population can be identified using the source samples.
    \begin{equation}
        \begin{aligned}[b]
        \mu_{a, T} & = E\left\{ Y_i{(a)} \mid R_i = T \right\} \\
        & = E\left\{ E\left\{ E\left\{ Y_i  \mid {X}_{i} = x, A_i = a, R_i = k\right\} \mid V_{i} = v, R_i = k \right\} \mid R_i = T \right\}. 
        \end{aligned}
        \label{equation_source_identification}
    \end{equation}
\end{theorem}

\section{Site-specific Estimators}
For the target site $k = \{T\}$, a standard AIPW estimator is used for $\mu_{a, T}$ as follows
\begin{equation}
    \widehat{\mu}_{a, T} = \frac{1}{n_T} \sum_{i = 1}^{n} \biggl[ \frac{I(A_i=a, R_i=T)}{\widehat{\pi}_{a, T}(V_{i})} \Bigl\{ Y_i-\widehat{m}_{a, T}(V_{i}) \Bigr\} + \widehat{m}_{a, T}(V_{i})\biggr],
    \label{equation_target_specific}
\end{equation}
where $\widehat{m}_{a, T}(V_{i})$ is an estimator for $E\left\{ Y_i \mid V_{i}=v, A_i = a, R_i=T\right\}$, the outcome model in the target population, and $\widehat{\pi}_{a, T}(V_{i})$ is an estimator for $P(A_i = 1 \mid V_{i} = v, R_i = T)$, the probability of receiving treatment $a$ in the target population. 

For each source site $k \in \mathcal{S}$, we propose an estimator for $\mu_{a, T}$ as follows
\begin{equation}
    \begin{aligned}[b]
    \widehat{\mu}_{a, k} 
    = & \frac{1}{n_k} \sum_{i = 1}^{n} \biggl[ \frac{I(A_i=a, R_i=k)}{\widehat{\pi}_{a, k}(X_{i})} \widehat{\zeta}_{k}(V_{i}) \Bigl\{ Y_i-\widehat{m}_{a, k}(X_{i}) \Bigr\} \biggr] \\
    + & \frac{1}{n_k} \sum_{i = 1}^{n} \biggl[ I(R_i=k) \widehat{\zeta}_{k}(V_{i})
    \Bigl\{ \widehat{m}_{a, k}(X_{i})-\widehat{\tau}_{a, k}(V_{i}) \Bigr\} \biggr] + \frac{1}{n_T} \sum_{i = 1}^{n} I(R_i = T) \widehat{\tau}_{a, k}(V_{i}),
    \end{aligned}
    \label{equation_site_specific}
\end{equation}
where $\widehat{\tau}_{a, k}(V_{i})$ is an estimator for $E\left\{ m_{a, k}(x) \mid V_{i}=v, R_i = k\right\}$ and $\widehat{m}_{a, k}(X_{i})$ is an estimator for $E\left\{ Y_i \mid X_{i}=x, A_i = a, R_i=k\right\}$. $\widehat{\zeta}_k(V_{i})$ estimates $f(V_{i} \mid R_i = T) / f(V_{i} \mid R_i = k)$, the density ratios of covariate distributions in the target population $T$ and source population $k \in \mathcal{S}$. $\widehat{\pi}_{a, k}(X_{i})$ estimates $P(A_i = 1 \mid X_{i} = x, R_i = k)$, the probability of receiving treatment $a$ in source $k \in \mathcal{S}$. 

Compared to the transportation estimators in \cite{dahabreh2019generalizing, dahabreh2020extending}, we introduce two additional nuisance functions, $\zeta_k(V_{i})$ and ${\tau}_{a, k}(V_{i})$. Specifically, $\zeta_k(V_{i})$ accounts for covariate shift across sites, while ${\tau}_{a, k}(V_{i})$ is introduced to address covariate mismatch across sites. We provide estimation procedures for these nuisance functions in the following subsections, and the theoretical guarantees of the estimator are presented in Section \ref{sec_theory}.

\subsection{Density Ratio Weighting}

Most existing methods for adjusting for heterogeneity of site populations rely on inverse probability of selection weighting, which requires pooling target and source samples. However, such pooling is often restricted to protect individuals' data privacy. We propose a density ratio weighting approach, which offers equivalent estimation without the need for direct data pooling (see Appendix \ref{appendix_equivalence_density_ratio}). 

Formally, we model the density ratios of covariate distributions in the target $T$ and source $k \in \mathcal{S}$ by specifying an exponential tilt model \cite{qin1998inferences, duan2020learningb}; ${\zeta}_k ( V_{i}; \gamma_k) = {f(V_{i} \mid R_i = T)} / {f(V_{i} \mid R_i = k)} = \exp \left\{-{\gamma}_k^{\top} \psi({V_{i}})\right\}$ where $f(V_{i} \mid R_i = T)$ and $f(V_{i} \mid R_i = k)$ are density functions of covariates $V_{i}$ in the target $T$ and source $k \in \mathcal{S}$, respectively, and $\psi({V_{i}})$ is some $d$-dimensional basis with $1$ as its first element. With this formulation, ${\zeta}_k ( V_{i}; \gamma_k) = 1$ for $\gamma_k = 0$ and $\int {\zeta}_k ( V_{i}; \gamma_k) {f(V_{i} \mid R_i = k)} dx = 1$. If we choose $\psi({V_{i}}) = V_{i}$, we can recover the entire class of natural exponential family distributions. If we include higher-order terms, the exponential tilt model has greater flexibility in characterizing the heterogeneity between two populations \cite{duan20201fast}. We solve for $\widehat{\gamma}_k$ with the following estimating equation 
\begin{equation}
    \label{equation_density_ratio}
    \frac{1}{n_T} \sum_{i=1}^N I\left(R_i = T\right) \psi\left(V_{i}\right)= \frac{1}{n_k} \sum_{i=1}^N I\left(R_i=k\right) \psi\left(V_{i}\right) \exp \left\{-{\gamma}_k^{\top} \psi({V_{i}})\right\}.
\end{equation}
This procedure preserves individual privacy; choosing $\psi({V_{i}}) = V_{i}$, the target site only needs to share its covariate means with the source sites; each source site then solves \eqref{equation_density_ratio} with its own data to obtain the density ratios.

\subsection{Multiply Robust Estimator}

We relax the assumption of homogeneous model specifications across sites and allow each site to propose multiple models for nuisance functions. Our proposal follows the construction of multiply robust estimators for nuisance functions via a model-mixing approach \cite{li2020demystifying}. 

Formally, for each site $k \in \mathcal{K}$, we consider a set of $J$ candidate treatment models for the propensity scores $\{\pi^{j}_{a, k}\left(x\right): j \in \mathcal{J} = \{1, ..., J \} \}$. Let $\widehat{\pi}_{a, k}^j(x)$ be the estimator of $\pi^{j}_{a, k}(x)$ obtained by fitting the corresponding candidate models on the data, which can be parametric, semiparametric, or nonparametric machine learning models. $\widehat{\pi}_{a, k}(X_i) = \sum_{j =1}^{J} \widehat{\Lambda}_j \widehat{\pi}^j_{a, k}(X_i)$ denotes the weighted predictions of propensity scores, with weights $\widehat{\Lambda}_j$ assigned to predictions by each candidate model $j$. To calculate the weights $\widehat{\Lambda}_j$, we adapt a model-mixing algorithm developed in \cite{yang2000adaptive} and \cite{li2020demystifying} based on the cumulative predictive risks of candidate models. 

First, we randomly partition the data within each site into a training set $D^\text{train}_k$ of units indexed by $\{1, ..., n_k^\text{train}\}$ and a validation set $D^\text{val}_k$ of units indexed by $\{n_k^\text{train} + 1, ..., n_{k}\}$. Then, each candidate treatment model is fit on $D^\text{train}_k$ to obtain $\widehat{\pi}_{a, n_k^\text{train}}^j$ for $j \in \mathcal{J}$. The model-mixing weights are determined by the models' predictive risks assessed on $D^\text{val}_k$ according to the Bernoulli likelihood. Specifically, 
\begin{equation}
    \begin{aligned}[b]
    \widehat{\Lambda}_j &= \left(n_k - n_k^\text{train}\right)^{-1} \sum_{i=n_k^\text{train}+1}^ {n_k} \widehat{\Lambda}_{j, i} \quad \text{and}  \\ 
    \widehat{\Lambda}_{j, i} &= \frac{\Pi_{q=n_k^\text{train}+1}^{i-1} \widehat{\pi}_{a, n_k^\text{train}}^j\left(X_q\right)^{A_q}\left\{1-\widehat{\pi}_{a, n_k^\text{train}}^j\left(X_q\right)\right\}^{1-A_q}}{\sum_{j^{\prime} = 1}^J \Pi_{q=n_k^\text{train}+1}^{i-1} \widehat{\pi}_{a, n_k^\text{train}}^{j^{\prime}}\left(X_q\right)^{A_q}\left\{1-\widehat{\pi}_{a, n_k^\text{train}}^{j^{\prime}}\left(X_q\right)\right\}^{1-A_q}} \quad \text{for} \quad n_k^\text{train}+2 \leq i \leq n_k,
    \end{aligned}
    \label{equation_mix_propensity}
\end{equation}

where $\widehat{\Lambda}_{j, n_k^{\text{train}}+1}=1/J$.
The model mixing estimators are consistent if one of the $j \in \mathcal{J}$ candidate models is correctly specified \cite{li2020demystifying}. A similar strategy extends for conditional outcomes $m_{a,k}(X_i)$ by combining a set of $L$ candidate outcome models $\{m^{l}_{a,k}\left( x\right): l\in \mathcal{L} =\{1, ..., L\} \}$. We obtain $\widehat{m}_{a,k}(X_i) = \sum_{l =1}^{L} \widehat{\Omega}_l \widehat{m}_{a,k}^l(X_i)$ as the predicted outcomes with weights $\widehat{\Omega}_l$ of candidate outcomes models under treatment $a$ in site $k$. Further details are provided in Appendix \ref{appendix_multiply_robust_outcome}.

\subsection{Handling Covariate Mismatch}
To account for covariate mismatch, we adapt the approach in \cite{zeng2023efficient}, introducing the nuisance function $\tau_{a, k}(V_{i}) = E\{m_{a, k}(x) \mid V_{i} = v, R_i = k\}$, where $m_{a, k}(x)$ is the outcome regression for treatment $a$ in site $k$. First, we estimate $m_{a, k}(X_i)$ by regressing the outcome $Y_i$ on covariates $X_{i}$ among units receiving treatment $a$ in site $k$. We then regress $\widehat{m}_{a, k}(X_i)$, the estimates from the previous step, on $V_{i}$ in the source site $k$ to obtain $\widehat{\tau}_{a, k}(x)$. By doing so, we project all site-specific estimates of conditional outcomes to a common hyperplane defined by $V_{i}$. If all effect modifiers that are distributed differently between target and source populations are measured in $V_{i}$, then the information contained in the projected site-specific estimates can be transported to the target site. Finally, we take the mean of $\widehat{\tau}_{a, k}(x)$ over the target sample, which gives us the transported estimate $\hat{\tau}_{a,k}(V_i)$ for the mean counterfactual outcomes under treatment $a$ in the target population.

\section{Federated Global Estimator}
\label{federated}
Let $\widehat{\mu}_{a, T}$ denote the estimate of $\mu_{a,T}$ based on target data only and $\widehat{\mu}_{a,k}$ be the estimates of $\mu_{a,T}$ using source data $k \in \mathcal{S}$. We propose a general form of the federated global estimator as follows 
\begin{equation}
    \widehat{\mu}_{a, G} 
    = \widehat{\mu}_{a, T} + \sum_{k \in \mathcal{K} } \widehat{\eta}_{k} \left\{ \widehat{\mu}_{a, k} - \widehat{\mu}_{a, T} \right\},
    \label{equation_fed_estimator}
\end{equation}
where $\widehat{\eta}_k \geq 0$ is a non-negative weight assigned to site-specific estimates and $\sum_{k \in \mathcal{K}} \widehat{\eta}_k = 1$. The role of $\eta_k$ is to determine the ensemble weight given to the site-specific estimates. We can employ diverse weighting methods by selecting appropriate values of $\eta_k$. For example, if $\eta_k = 0$, the global estimator is simply the estimator based on target data only; if $\eta_k = n_k/N$, the global estimator combines site-specific estimates by their sample sizes; if $\eta_k = (1 / \sigma_k^2)/ \sum_{j \in \mathcal{K}} (1 / \sigma_j^2) $ where $\sigma_k^2 = \text{Var}(\widehat{\mu}_{a, k})$, the global estimator is the inverse variance weighting estimator, which is known to be appropriate when working models are homogeneous across sites
\cite{xiong2021federated}. To control for bias due to non-transportable site estimates while achieving optimal efficiency, we estimate $\eta_k$ data-adaptively by a penalized regression of site-specific influence functions \cite{han2021federated, han2022privacy}. This strategy ensembles the site-specific estimates for higher efficiency if they are sufficiently similar to the target estimates; if source estimates are significantly different, their weights will be shrunk toward zero with high probability. 

We denote the data-adaptive weights as $\eta_{k, L_1}$, obtained as the solutions to a penalized regression of the site-specific influence functions as follows
\begin{equation}
   \label{equation_l1_weights}
   \widehat{\eta}_{k, L_1} = \arg \min_{\eta_k \geq 0} \sum_{i=1}^N\left[\widehat{\xi}_{T, i}{(a)}-\sum_{k \in \mathcal{K} } \eta_k\left(\widehat{\xi}_{T, i}{(a)}-\widehat{\xi}_{k, i}{(a)}-\widehat{\delta}_k\right)\right]^2  + \lambda \sum_{k \in \mathcal{K} }\left|\eta_k\right| \widehat{\delta}_k^2,
\end{equation}
where $\widehat{\xi}_{T, i}{(a)}$ and $\widehat{\xi}_{k, i}{(a)}$ are the estimated influence functions for the target and source site estimators (see Appendix \ref{appendix_site_estimator_influence} for the exact form of the influence functions). The estimated difference $\widehat{\delta}_k = \widehat{\mu}_{a, k} - \widehat{\mu}_{a, T}$ quantifies the bias between the estimate from source $k \in \mathcal{S}$ and the estimate from the target $T$.  The tuning parameter $\lambda$ determines the penalty imposed on source site estimates and in practice, is chosen via cross-validation. Specifically, we create a grid of values of $\lambda$ and iteratively train and evaluate the model using different $\lambda$ values, selecting the one with the lowest average validation error after multiple sample splits.

We estimate the variance of $\widehat{\mu}_{a, G}$ using the estimated influence functions for $\widehat{\mu}_{a, T}$ and $\widehat{\mu}_{a, k}$. 
By the central limit theorem, $\sqrt{N} (\widehat{\mu}_{a, G} -\bar{\mu}_{a, G} ) \stackrel{d}{\rightarrow} \mathcal{N}(0, \Sigma)$, where $\Sigma=E\{\sum_{k\in \mathcal{K}} \bar{\eta}_k \xi_{k, i}(a)\}^2$ and $\bar{\mu}_{a, G}$ and $\bar{\eta}_k$ denote the limiting values of $\widehat{\mu}_{a, G}$ and $\hat{\eta}_k$ respectively. The standard error of $\widehat{\mu}_{a, G}$ is estimated as $ \sqrt{\widehat{\Sigma}/N}$ where $\widehat{\Sigma}=N^{-1} \sum_{k\in \mathcal{K}} \sum_{i=1}^{n_k}\{\widehat{\eta}_k \widehat{\xi}_{k, i}{(a)}\}^2$. A two-sided $(1-\alpha)\times$100\% confidence interval for $\mu_{a, G}$ is
\begin{equation}
    \widehat{\mathcal{C}}_\alpha=\left[\widehat{\mu}_{a, G}-\sqrt{\widehat{\Sigma}/N} \mathcal{Z}_{\alpha/2}, \quad \widehat{\mu}_{a, G}+\sqrt{\widehat{\Sigma} / N} \mathcal{Z}_{\alpha/2}\right],
    \label{equation_confidence_interval}
\end{equation}
where $\mathcal{Z}_{\alpha/2}$ is the $1 - \alpha / 2$ quantile for a standard normal distribution.

\section{Theoretical Guarantees}
\label{sec_theory}
\subsection{Site-specific Estimator}
We first establish the theoretical properties of the site-specific estimators constructed with the multiply robust model-mixing approach. Define $\overline{\pi}_{a, k}^j$, $\overline{m}_{a,k}^l$, $\overline{\tau}_{a, k}$ and $\overline{\zeta}_k$ as non-stochastic functionals that the corresponding estimators $\widehat{\pi}_{a, k}^j$, $\widehat{m}_{a,k}^l$, $\widehat{\tau}_{a, k}$ and $\widehat{\zeta}_k$ converge to. That is, 
\begin{align}
    \| \widehat{\pi}_{a, k}^j - \overline{\pi}_{a, k}^j \| & = o_p(1), \quad
    \| \widehat{m}_{a, k}^l - \overline{m}_{a, k}^l \|  = o_p(1), \nonumber \quad
    \| \widehat{\tau}_{a, k} - \overline{\tau}_{a, k} \|  = o_p(1), \quad
    \| \widehat{\zeta}_{k} - \overline{\zeta}_{k} \| & = o_p(1). \nonumber 
\end{align}

As shown in Lemmas \ref{lemma1_mix_propensity_risk} and \ref{lemma2_mix_outcome_risk} in Appendix \ref{appendix_li_lemmas}, the $L_2$ risks of the model mixing estimators $\widehat{\pi}_{a, k}$ and $\widehat{m}_{a, k}$ are bounded by the smallest risks of all candidate models plus a remainder term that vanishes at a faster rate than the risks themselves. 

\begin{theorem}
\label{theorem_site_consistency_normality}
Suppose that the conditions in Lemmas \ref{lemma1_mix_propensity_risk} and \ref{lemma2_mix_outcome_risk} hold, and that $\widehat{\pi}^j_{a, k}$, $\widehat{m}^l_{a, k}$, $\widehat{\zeta}_k$, $\widehat{\tau}_{a, k}$, $\bar{\pi}^j_{a, k}$, $\bar{m}^l_{a, k}$, $\bar{\zeta}_{k}$ and $\bar{\tau}_{a, k}$ are uniformly bounded for all treatment models $j \in\mathcal{J}$ and for all outcome models $l \in \mathcal{L}$. Consider the following conditions: 
\begin{enumerate}[start = 2, label=]
    \item
    \begin{enumerate}[label=(\Alph{enumi}\arabic*)]
    \item $\overline{\pi}_{a, k}^j = \pi_{a, k}$ for some $j \in \mathcal{J}$, \label{assumption_multiply_robust_B1}
    \item $\overline{m}_{a, k}^l = m_{a, k}$ for some $l \in \mathcal{L}$,\label{assumption_multiply_robust_B2}
    \end{enumerate}
    \item
    \begin{enumerate}[label=(\Alph{enumi}\arabic*)]
    \item $\overline{\zeta}_{k} = \zeta_{k}$, \label{assumption_multiply_robust_C1}
    \item $\overline{\tau}_{a, k} = \tau_{a, k}$. \label{assumption_multiply_robust_C2}
    \end{enumerate}
\end{enumerate}
Then, under Assumptions \ref{assumption_consistency} - \ref{assumption_selection_positivity}, and if one of \ref{assumption_multiply_robust_B1} or \ref{assumption_multiply_robust_B2} and one of \ref{assumption_multiply_robust_C1} or \ref{assumption_multiply_robust_C2} hold,  
\begin{align}
    \| \widehat{\mu}_{a, k}-\mu_{a, T} \| = O_p\left(n^{-1 / 2} + \|\widehat{\pi}_{a, k}-\pi_{a, k} \| \|\widehat{m}_{a, k}-{m}_{a, k} \|
    + \|\widehat{\zeta}_{k}-{\zeta}_{k} \| \|\widehat{\tau}_{a, k}-\tau_{a, k} \|
    \right).
\end{align}
Further, if the nuisance estimators satisfy the following convergence rate
\begin{align}
    \left\|\widehat{m}_{a, k}-m_{a, k} \right\|\left\|\widehat{\pi}_{a, k}-\pi_{a, k} \right\| &= o_p(1 / \sqrt{n}), \nonumber \quad
    \|\widehat{\zeta}_k -\zeta_k \|\left\|\widehat{\tau}_{a, k}-\tau_{a, k} \right\| =o_p(1 / \sqrt{n}), \nonumber
\end{align}
then $\sqrt{n} ( \widehat{\mu}_{a,k} - \mu_{a, T} )$
asymptotically converges to a normal distribution with mean zero and asymptotic variance equal to the semiparametric efficiency bound. The derivation of the result is provided in the Appendix.
\end{theorem}

\subsection{Federated Global Estimator}
\begin{theorem}
    \label{theorem_global_consistency_normality}
    Under Assumptions \ref{assumption_consistency} - \ref{assumption_selection_positivity} and the regularity conditions specified in the Appendix, the federated global estimator of $\Delta_T$, given by $\widehat{\Delta}_{G} = \widehat{\mu}_{1, G} - \widehat{\mu}_{0, G}$, is consistent and asymptotically normal,
    \begin{equation}
        \sqrt{N / \widehat{\mathcal{V}}}\left(\widehat{\Delta}_{G}-\Delta_{T}\right) \stackrel{d}{\rightarrow} \mathcal{N}(0,1),
    \end{equation}
with the variance estimated consistently as $\widehat{\mathcal{V}}$. The variance of $\widehat{\Delta}_{G}$ is no larger than that of the estimator based on target data only, $\widehat{\Delta}_{T} = \widehat{\mu}_{1, T} - \widehat{\mu}_{0, T}$. Further, if there exist some source sites with consistent estimators of $\Delta_{T}$ and satisfy conditions specified in the Appendix, the variance of $\widehat{\Delta}_{G}$ is strictly smaller than $\widehat{\Delta}_{T}$. 
\end{theorem}

\section{Experiments}
\label{section_experiments}
We evaluate the finite sample properties of five different estimators: (i) an augmented inverse probability weighted (AIPW) estimator using data from the target site only (Target), (ii) an AIPW estimator that weights each site proportionally to its sample size (SS), (iii) an AIPW estimator that weights each site inverse-proportionally to its variance (IVW), (iv) an AIPW estimator that weights each site with the $L_1$ weights defined in \eqref{equation_l1_weights} (AIPW-$L_1$), and (v) a multiply robust estimator with the $L_1$ weights defined in \eqref{equation_l1_weights} (MR-$L_1$).

Across different settings, we examine the performance of each estimator in terms of bias, root mean square error, and coverage and length of 95\% confidence intervals (CI) across $500$ simulations.

We consider a total of five sites and fix the first site as the target site with a relatively small sample size of $300$. The source sites have larger sample sizes of $\left\{500, 500, 1000, 1000 \right\}$. We model heterogeneity in the covariate distributions across sites with skewed normal distributions and varying levels of skewness in each site, $X_{kp} \sim \mathcal{S N}\left(x ; \Xi_{k p}, \Omega_{k p}^2, \mathrm{~A}_{k p}\right)$, where $k \in \{1, ..., 5\}$ indexes each site and $p \in \{1, ..., 4\}$ indexes the covariates; $\Xi_{k p}$, $\Omega_{k p}^2$ and $\mathrm{~A}_{k p}$ are the location, scale, and skewness parameters, respectively. Following \cite{kangschafer2007data-generation}, we also generate  covariates $Z_{kp}$ as non-linear transformation of $X_{kp}$ such that $Z_{k1} = \exp (X_{k1} / 2)$, $Z_{k2} = X_{k2} / \{1+\exp (X_{k1}) \}+10$, $Z_{k3}=(X_{k1} X_{k3} / 25+0.6)^3$ and $Z_{k4}=(X_{k2}+X_{k4}+20)^2$.

For the MR-$L_1$, we adaptively mix two outcome models and two treatment models. We specify the first model with the covariates $X_{kp}$, and the second model with the covariates $Z_{kp}$.  The AIPW-$L_1$ estimator requires a common model to be specified across sites, so we specify the outcome and treatment models using covariates $X_{kp}$. 

The tuning parameter $\lambda$ is selected through cross-validation using a grid of values $\{0,10^{-3},10^{-2},0.1,0.5,1,2,5,10\}$. To perform cross-validation, the simulated datasets in each site are split into two equally sized training and validation datasets.

\subsection{No Covariate Mismatch}

We first consider the setting where there is no covariate mismatch, i.e. $p = 4$ for both target and source sites. For each unit, we generate potential outcomes as
\begin{equation}
    Y_k(a) = 210 + X_k \beta_x + Z_k \beta_z + \varepsilon_k
    \label{equation_DGP_outcome}
\end{equation}
where $\beta_x = \beta_z = (27.4, 13.7, 13.7, 13.7)$. For units in the target site, we generate outcomes with $X_k$ only by setting $\beta_z = 0$; for units in the source sites, either $X_{k}$ or $Z_{k}$ is used to generate outcomes. If $\beta_x \neq 0$, then $\beta_z = 0$ and vice versa. Similarly, the treatment is generated as 
\begin{equation}
    A_k \sim \operatorname{Bernoulli}\left(\pi_k \right) \quad \pi_k=\operatorname{expit}(X_k \alpha_x + Z_k \alpha_z )
    \label{equation_DGP_treatment}
\end{equation}
where 
$\alpha_x = \alpha_z = (-1, 0.5, -0.25, -0.1)$. For units in the target site, we generate treatments with $X_k$ only by setting $\alpha_z = 0$; for units in the source sites, either $X_{k}$ or $Z_{k}$ is used to generate treatments. If $\alpha_x \neq 0$, then $\alpha_z = 0$ and vice versa. With this data generation scheme, the true ATE is $\Delta_{T} = 0$. 

We compare the performance of the five estimators described above under the following settings: 

\textbf{Setting 1 ($C = 0$)}: outcomes and treatments in all source sites are generated with $Z_k$. However, all source sites misspecify both models with $X_k$. The target site correctly specifies both models. 

\textbf{Setting 2 ($C = 1/2$)}: outcomes and treatments are generated with $X_k$ in Sites 2 and 4, but with $Z_k$ in Sites 3 and 5; thus, the outcome and treatment models are misspecified in Sites 3 and 5, and only half of the source sites correctly specify the models. 

\textbf{Setting 3 ($C = 1$)}: outcomes and treatments in all source sites are generated with $X_k$, so all source sites correctly specify outcome and treatment models with $X_k$. 

\begin{table}[h]
\caption{Mean absolute error (MAE), root mean squared error (RMSE), coverage (Cov.), and length (Len.) of $95\%$ CIs based on 500 simulated data sets in three (mis)specification settings.}
\label{table_varying_prop_correct_sites}
\centering
\footnotesize
\begin{tabular}{lrrrrr} 
  \toprule
  & Target & SS & IVW & AIPW-$L_1$ & MR-$L_1$ \\
  \cmidrule{2-6}  
  $C = 0$ &  &  & \\ 
  \hspace{.2cm} MAE &  0.109 & 1.933 & 0.177 & 0.110  & 0.050 \\
  \hspace{.2cm} RMSE & 0.141 & 1.987 & 0.219 & 0.144 & 0.061 \\ 
  \hspace{.2cm} Cov. & 0.950 & 0.998 & 0.826 & 0.936 & 0.960 \\ 
  \hspace{.2cm} Len. & 0.551 & 7.035 & 0.567 & 0.547 & 0.234 \\
  \midrule
  $C = 1/2$ &  &  &\\ 
  \hspace{.2cm} MAE &  0.109 & 1.111 & 0.107 & 0.109 &  0.050 \\
  \hspace{.2cm} RMSE & 0.141 & 1.189 & 0.139 & 0.140  & 0.062 \\ 
  \hspace{.2cm} Cov. & 0.950 & 1.000 & 0.942 & 0.950  & 0.962 \\ 
  \hspace{.2cm} Len. & 0.551 & 6.010 & 0.540 & 0.547 &  0.242 \\
  \midrule
  $C = 1$ &  &  &\\ 
  \hspace{.2cm} MAE &  0.109 & 0.036 & 0.035 & 0.050 &  0.049 \\
  \hspace{.2cm} RMSE & 0.141 & 0.045 & 0.044 & 0.064 &  0.063 \\ 
  \hspace{.2cm} Cov. & 0.950 & 0.968 & 0.956 & 0.958 & 0.960 \\ 
  \hspace{.2cm} Len. & 0.551 & 0.195 & 0.191 & 0.260 & 0.253 \\ 
  \bottomrule
\end{tabular}
\end{table}

The results in Table \ref{table_varying_prop_correct_sites} indicate that the MR-$L_1$ estimator has lower RMSE than the Target estimator when some source sites have correctly specified models ($C = 1/2$ and $C = 1$). Relative to the MR-$L_1$ estimator, the SS and IVW estimators demonstrate larger biases and RMSE, and lower coverage when some source sites have misspecified models ($C=0$ and $C = 1/2)$. The MR-$L_1$ estimator shows reduced biases and RMSE compared to the AIPW-$L_1$ estimator, while maintaining similar coverage; this improvement can be attributed to the inclusion of an additional model that closely resembles the true model. When all source sites correctly specify working models ($C=1$), the IVW estimator performs optimally with the shortest confidence interval as expected. 

\subsection{Covariate Mismatch}

To demonstrate that our proposed MR-$L_1$ estimator can handle covariate mismatch across sites, 
we modify the data-generating process in the following way: only two covariates are used in the outcome and treatment generation processes in the target site. Specifically, the generating models remain the same as in \eqref{equation_DGP_outcome} and \eqref{equation_DGP_treatment}, using covariates $X_{k}$. However, for units in the target site, we set $\beta_x = (27.4, 13.7, 0, 0)$ for outcome generation and $\alpha_x = (-1, 0.5, 0, 0)$ for treatment generation.

The AIPW-$L_1$ estimator, which requires common models across sites, only uses the shared covariates ($X_{k1}$ and $X_{k2}$) to specify outcome and treatment models for all sites. On the other hand, our MR-$L_1$ estimator allows for different covariates in different sites, so we utilize both shared covariates with the target site and unique covariates to specify the outcome and treatment models in the source sites.

\begin{table}[h]
\caption{Mean absolute error (MAE), root mean squared error (RMSE), coverage (Cov.), and length (Len.) of $95\%$ CIs based on 500 simulated data sets in covariate mismatch settings.}
\label{table_covariate_mismatch}
\centering
\footnotesize
\begin{tabular}{lrrrrr} 
  \toprule
  & Target & SS & IVW & AIPW-$L_1$ & MR-$L_1$ \\
  \cmidrule{2-6} 
  \hspace{.1cm} MAE &  0.108 & 4.331 & 0.150 & 0.107 &  0.053 \\
  \hspace{.1cm} RMSE & 0.136 & 4.401 & 0.186 & 0.134 &  0.067 \\ 
  \hspace{.1cm} Cov. & 0.946 & 1.000 & 0.882 & 0.950 &  0.944 \\ 
  \hspace{.1cm} Len. & 0.538 & 26.024 & 0.553 & 0.536 &  0.253 \\
  \bottomrule
\end{tabular}
\end{table}

In Table \ref{table_covariate_mismatch}, we observe that the AIPW-$L_1$ estimator exhibits similar bias, RMSE, coverage, and length of confidence intervals as the Target estimator while outperforming the SS and IVW estimators. This is because relying solely on shared covariates leads to significant biases in all source sites (Figure \ref{figure_site-specific}, left panel), and the AIPW-$L_1$ estimator assigns nearly all of the ensemble weight to the target site so as to reduce bias.

In contrast, the MR-$L_1$ estimator outperforms the AIPW-$L_1$ estimator by exhibiting substantially smaller bias, lower RMSE, and better coverage. This improvement can be attributed to the inclusion of unique covariates from the source sites, which allows for the recovery of true models in those sites and contributes to a more accurate estimation of $\Delta_{T}$ (Figure \ref{figure_site-specific}, right panel). These findings suggest that neglecting covariate mismatch by solely relying on shared covariates can lead to highly biased results.

\begin{figure}[!htbp]
\centering
\includegraphics[width=0.75\textwidth]{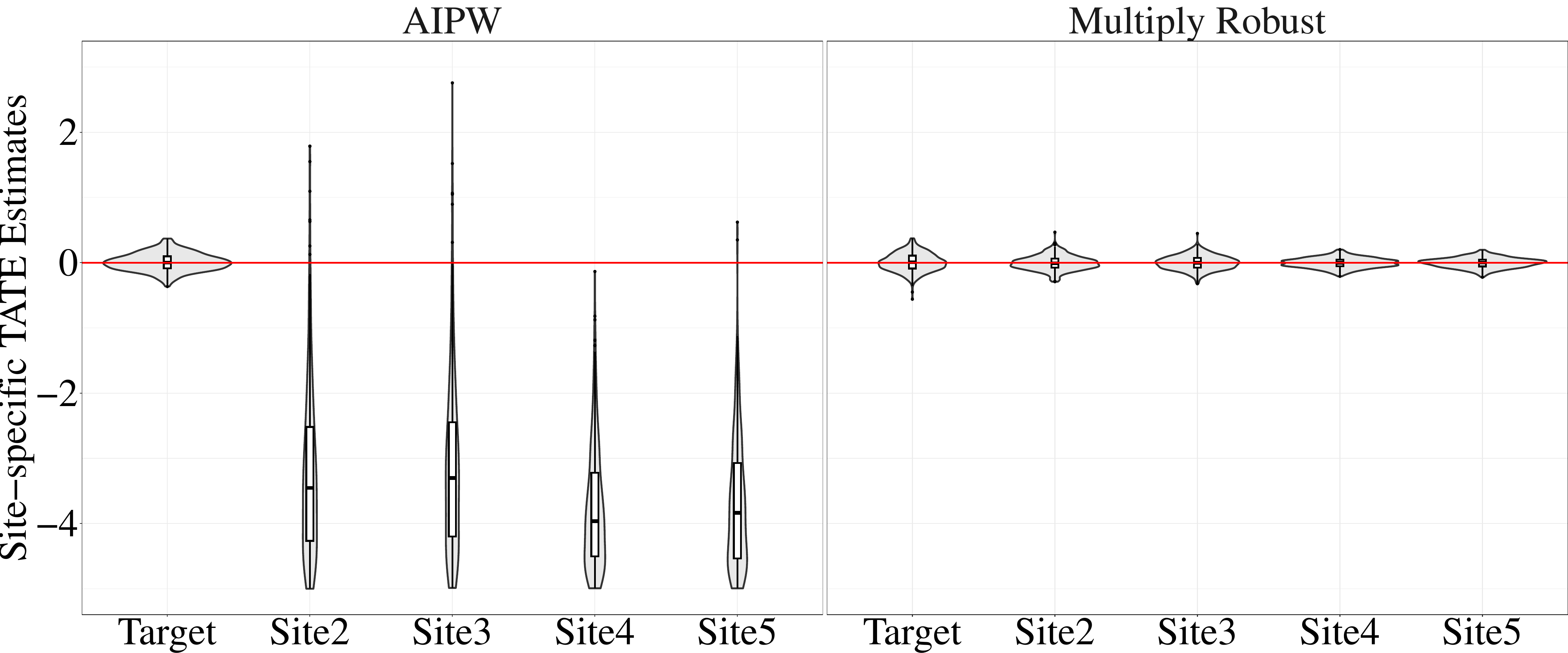}
\caption{Estimates of the TATE based on $500$ simulated data sets with covariate mismatch comparing the site-specific estimators with nuisance functions estimated by AIPW (left) and by multiply robust model-mixing (right).} \label{figure_site-specific}
\end{figure}

\section{Conclusion}
\label{section_discussion}
We have proposed a novel federated approach for \textit{privacy-preserving}, \textit{multiply robust}, and \textit{flexible} estimation of causal effects. Compared to existing federated methods, our proposed approach accommodates covariate shift and covariate mismatch across sites, while guaranteeing efficient estimation and preserving privacy in the sense that only covariate means of the target samples are shared in a single round of communication. Our proposal allows investigators in each site to have greater flexibility in specifying candidate models by utilizing site-specific information. Moreover, our method utilizes adaptive ensemble weights to avoid negative transfer in the federation process. The limitations of the current proposal provide opportunities for further research. To handle high-dimensional covariates, future research can explore ways to jointly model the propensity score and density ratio to reduce the dimension of parameters for population balancing. 
   
\bibliographystyle{abbrv}

\bibliography{neurips_2023}

\begin{thebibliography}{10}

\bibitem{buchanan2018generalizing}
A.~L. Buchanan, M.~G. Hudgens, S.~R. Cole, K.~R. Mollan, P.~E. Sax, E.~S. Daar,
  A.~A. Adimora, J.~J. Eron, and M.~J. Mugavero.
\newblock Generalizing evidence from randomized trials using inverse
  probability of sampling weights.
\newblock {\em Journal of the Royal Statistical Society. Series A,(Statistics
  in Society)}, 181(4):1193, 2018.

\bibitem{chan2013simple}
K.~C.~G. Chan.
\newblock A simple multiply robust estimator for missing response problem.
\newblock {\em Stat}, 2(1):143--149, 2013.

\bibitem{chan2014oracle}
K.~C.~G. Chan and S.~C.~P. Yam.
\newblock Oracle, multiple robust and multipurpose calibration in a missing
  response problem.
\newblock {\em Statistical Science}, 29(3):380--396, 2014.

\bibitem{chen2017multiply}
S.~Chen and D.~Haziza.
\newblock Multiply robust imputation procedures for the treatment of item
  nonresponse in surveys.
\newblock {\em Biometrika}, 104(2):439--453, 2017.

\bibitem{cole2010generalizing}
S.~R. Cole and E.~A. Stuart.
\newblock Generalizing evidence from randomized clinical trials to target
  populations: the actg 320 trial.
\newblock {\em American journal of epidemiology}, 172(1):107--115, 2010.

\bibitem{dahabreh2020extending}
I.~J. Dahabreh, S.~E. Robertson, J.~A. Steingrimsson, E.~A. Stuart, and M.~A.
  Hernan.
\newblock Extending inferences from a randomized trial to a new target
  population.
\newblock {\em Statistics in Medicine}, 39(14):1999--2014, 2020.

\bibitem{dahabreh2019generalizing}
I.~J. Dahabreh, S.~E. Robertson, E.~J. Tchetgen, E.~A. Stuart, and M.~A.
  Hern{\'a}n.
\newblock Generalizing causal inferences from individuals in randomized trials
  to all trial-eligible individuals.
\newblock {\em Biometrics}, 75(2):685--694, 2019.

\bibitem{duan2020learningb}
R.~Duan, C.~Luo, M.~H. Schuemie, J.~Tong, J.~C. Liang, H.~H. Chang, M.~R.
  Boland, J.~Bian, H.~Xu, J.~H. Holmes, et~al.
\newblock Learning from local to global-an efficient distributed algorithm for
  modeling time-to-event data.
\newblock {\em bioRxiv}, 2020.

\bibitem{duan20201fast}
R.~Duan, Y.~Ning, S.~Wang, B.~Lindsay, R.~Carroll, and Y.~Chen.
\newblock A fast score test for generalized mixture models.
\newblock {\em Biometrics}, 76:811--820, 2021.

\bibitem{Gu2019aggregate}
Y.~Gu and H.~Zou.
\newblock Aggregated expectile regression by exponential weighting.
\newblock {\em Statistica Sinica}, 29(2):671–692, 2019.

\bibitem{guo2021multi}
W.~Guo, S.~Wang, P.~Ding, Y.~Wang, and M.~I. Jordan.
\newblock Multi-source causal inference using control variates.
\newblock {\em arXiv preprint arXiv:2103.16689}, 2021.

\bibitem{han2021federated}
L.~Han, J.~Hou, K.~Cho, R.~Duan, and T.~Cai.
\newblock Federated adaptive causal estimation (face) of target treatment
  effects.
\newblock {\em arXiv preprint arXiv:2112.09313}, 2021.

\bibitem{han2022privacy}
L.~Han, Y.~Li, B.~A. Niknam, and J.~R. Zubizarreta.
\newblock Privacy-preserving and communication-efficient causal inference for
  hospital quality measurement.
\newblock {\em arXiv preprint arXiv:2203.00768}, 2022.

\bibitem{han2014further}
P.~Han.
\newblock A further study of the multiply robust estimator in missing data
  analysis.
\newblock {\em Journal of Statistical Planning and Inference}, 148:101--110,
  2014.

\bibitem{han2014multiply}
P.~Han.
\newblock Multiply robust estimation in regression analysis with missing data.
\newblock {\em Journal of the American Statistical Association},
  109(507):1159--1173, 2014.

\bibitem{han2013estimation}
P.~Han and L.~Wang.
\newblock Estimation with missing data: beyond double robustness.
\newblock {\em Biometrika}, 100(2):417--430, 2013.

\bibitem{kangschafer2007data-generation}
J.~D.~Y. Kang and J.~L. Schafer.
\newblock Demystifying double robustness: A comparison of alternative
  strategies for estimating a population mean from incomplete data.
\newblock {\em Statistical Science}, 22(4):523–539, 2007.

\bibitem{kawamata2022federated}
Y.~Kawamata, R.~Motai, Y.~Okada, A.~Imakura, and T.~Sakurai.
\newblock Collaborative causal inference on distributed data.
\newblock {\em arXiv preprint arXiv:2208.07898}, 2022.

\bibitem{kennedy2022minimax}
E.~H. Kennedy, S.~Balakrishnan, and L.~Wasserman.
\newblock Minimax rates for heterogeneous causal effect estimation.
\newblock {\em arXiv preprint arXiv:2203.00837}, 2022.

\bibitem{li2020demystifying}
W.~Li, Y.~Gu, and L.~Liu.
\newblock Demystifying a class of multiply robust estimators.
\newblock {\em Biometrika}, 107(4):919--933, 2020.

\bibitem{neyman1923application}
J.~Neyman.
\newblock On the application of probability theory to agricultural experiments.
\newblock {\em Statistical Science}, 5(5):463--480, 1923.

\bibitem{qin1998inferences}
J.~Qin.
\newblock Inferences for case-control and semiparametric two-sample density
  ratio models.
\newblock {\em Biometrika}, 85(3):619--630, 1998.

\bibitem{rubin1974estimating}
D.~B. Rubin.
\newblock Estimating causal effects of treatments in randomized and
  nonrandomized studies.
\newblock {\em Journal of Educational Psychology}, 66(5):688, 1974.

\bibitem{rubin1980randomization}
D.~B. Rubin.
\newblock Randomization analysis of experimental data: the fisher randomization
  test comment.
\newblock {\em Journal of the American Statistical Association},
  75(371):591--593, 1980.

\bibitem{tipton2013improving}
E.~Tipton.
\newblock Improving generalizations from experiments using propensity score
  subclassification: Assumptions, properties, and contexts.
\newblock {\em Journal of Educational and Behavioral Statistics},
  38(3):239--266, 2013.

\bibitem{vo2021federated}
T.~V. Vo, T.~N. Hoang, Y.~Lee, and T.-Y. Leong.
\newblock Federated estimation of causal effects from observational data.
\newblock {\em arXiv preprint arXiv:2106.00456}, 2021.

\bibitem{xiong2021federated}
R.~Xiong, A.~Koenecke, M.~Powell, Z.~Shen, J.~T. Vogelstein, and S.~Athey.
\newblock Federated causal inference in heterogeneous observational data.
\newblock {\em arXiv preprint arXiv:2107.11732}, 2021.

\bibitem{yang2020combining}
S.~Yang and P.~Ding.
\newblock Combining multiple observational data sources to estimate causal
  effects.
\newblock {\em Journal of the American Statistical Association},
  115(531):1540--1554, 2020.

\bibitem{yang2000adaptive}
Y.~Yang.
\newblock Adaptive estimation in pattern recognition by combining different
  procedures.
\newblock {\em Statistica Sinica}, pages 1069--1089, 2000.

\bibitem{yang2004ensemble}
Y.~Yang.
\newblock Combining forecasting procedures: Some theoretical results.
\newblock {\em Econometric Theory}, 20(1):176--222, 2004.

\bibitem{zeng2023efficient}
Z.~Zeng, E.~H. Kennedy, L.~M. Bodnar, and A.~I. Naimi.
\newblock Efficient generalization and transportation.
\newblock {\em arXiv preprint arXiv:2302.00092}, 2023.

\end{thebibliography}

\newpage
\pagenumbering{arabic}
\appendix

\noindent{\Large \bfseries APPENDIX \par}
\section{Equivalence of density ratio weighting and inverse of selection probability weighting}
\label{appendix_equivalence_density_ratio}

We show that the inverse probability of selection weighting (IPSW) to the site $k$, $\rho_k(V_i) = \frac{1 - P(R_i = k \mid V_i = v)}{P(R_i = k \mid V_i = v)}$ is equivalent to the density ratio weighting $ \zeta_k(V_i) = \frac{P(V_i = v \mid R_i = T)}{P(V_i = v \mid R_i = k)}$. 

Under Assumptions \ref{assumption_consistency} - \ref{assumption_selection_positivity}, the site-specific estimators based on the IPSW for $\mu_{a, T}$ is
\begin{equation}
    \begin{aligned}[b]
    \widehat{\mu}_{a, k} & = \frac{1}{n_T} \sum_{i = 1}^{n} \biggl[ \frac{I(A_i=a, R_i=k)}{\widehat{\pi}_{a, k}(X_{i})} \widehat{\rho}_k(V_i) \{ Y_i-\widehat{m}_{a, k}(X_{i}) \} \biggr] \\
    & + \frac{1}{n_T} \sum_{i = 1}^{n} \biggl[ I(R_i=k) \widehat{\rho}_k(V_i)
    \Bigl\{ \widehat{m}_{a, k}(X_{i})-\widehat{\tau}_{a, k}(V_i) \Bigr\} \biggr] \\
    & + \frac{1}{n_T} \sum_{i = 1}^{n} I(R_i=T) \widehat{\tau}_{a, k}(V_{i})
    \end{aligned}
    \label{equation_zeng_transport}
\end{equation}
where $\widehat{\rho}_k(V_i) = \frac{1 - \widehat{P}(R_i = k \mid V_i = v)}{\widehat{P}(R_i = k \mid V_i = v)}$ is an estimator for $\rho_k(V_i)$. Applying Baye's rule, we show that the IPSW is equivalent to the density ratio weighting up to a constant, 
\begin{equation}
    \begin{aligned}[b]
    \rho_k(V_i) &= \frac{1 - P(R_i = k \mid V_i = v)}{P(R_i = k \mid V_i = v)} \\ 
    &= \frac{P(R_i = T \mid V_i = v)}{P(R_i = k \mid V_i = v)} \\
    &= \frac{P(V_i = v \mid R_i = T) P(R_i = T)}{P(V_i = v \mid R_i = k) P(R_i = k)} \\
    &= {\zeta}_k (V_i) \frac{P(R_i = T)}{P(R_i = k)}. 
    \end{aligned}
    \label{equation_equivalence_density_ratio}
\end{equation}
We re-write \eqref{equation_zeng_transport} by 
substituting $\widehat{\rho}_k(V_i)$ with $\widehat{\zeta}_k(V_i)$, 
\begin{equation}
    \begin{aligned}[b]
\widehat{\mu}_{a, k} 
& = \frac{1}{n_T} \sum_{i = 1}^{n} \biggl[ \frac{I(A_i=a, R_i=k)}{\widehat{\pi}_{a, k}(X_{i})} {\widehat{\zeta}}_k (V_i) \frac{\widehat{P}(R_i = T)}{\widehat{P}(R_i = k)} \{Y_i-\widehat{m}_{a, k}(X_{i}) \} \biggr] \\
& + \frac{1}{n_T} \sum_{i = 1}^{n} \biggl[ I(R_i=k) {\widehat{\zeta}}_k (V_i) \frac{\widehat{P}(R_i = T)}{\widehat{P}(R_i = k)} \{ \widehat{m}_{a, k}(X_{i})-\widehat{\tau}_{a, k}(V_i) \} \biggr] \\
& + \frac{1}{n_T} \sum_{i = 1}^{n} I(R_i=T) \widehat{\tau}_{a, k}(V_{i}). 
\end{aligned}
\end{equation}
A reasonable estimator $\frac{\widehat{P}(R_i = T)}{\widehat{P}(R_i = k)}$ is $\frac{n_T}{n_k}$. Therefore, we recover our proposed site-specific estimator for $\mu_{a, T}$,  
\begin{equation}
\begin{aligned}[b]
\widehat{\mu}_{a, k} 
& = \frac{1}{n_k} \sum_{i = 1}^{n} \biggl[ \frac{I(A_i=a, R_i=k)}{\widehat{\pi}_{a, k}(X_{i})} {\widehat{\zeta}}_k (V_i) \{Y_i-\widehat{m}_{a, k}(X_{i})\} \biggr] \\
& + \frac{1}{n_k} \sum_{i = 1}^{n} \biggl[ I(R_i=k) {\widehat{\zeta}}_k (V_i) \{ \widehat{m}_{a, k}(X_{i})-\widehat{\tau}_{a, k}(V_i) \} \biggr] \\ 
& + \frac{1}{n_T} \sum_{i = 1}^{n} I(R_i=T) \widehat{\tau}_{a, k}(V_{i}). 
\end{aligned}
\end{equation}

\section{Multiply robust estimation for ${m}_{a, k}$}
\label{appendix_multiply_robust_outcome}
For multiply robust outcome estimation within each site $k$, we consider a set of $L$ candidate models for conditional outcomes $\{m^{l}_{a,k}\left( x\right): l \in \mathcal{L} = \{1, ..., L \} \}$. Let $\hat{m}_{a,k}^l(x)$ be the estimates of $m^{l}_{a,k}(x)$ obtained by fitting the corresponding candidate models, which can be parametric, semiparametric, or nonparametric machine learning models. Let $\widehat{m}_{a,k}(X_i) = \sum_{l =1}^{L} \widehat{\Omega}_l \widehat{m}_{a,k}^l(X_i)$ be the predictions with ensemble weights $\widehat{\Omega}_l$ of candidate outcome models under treatment $a$ in site $k$. We derive the ensemble weights $\widehat{\Omega}_l$ based on the cumulative predictive risk of candidate models. In particular, we denote the data corresponding to treated and control units as $D_{k, 1}$ and $D_{k, 0}$ respectively, and the sample sizes of $D_{k, 1}$ and $D_{k, 0}$ are denoted as $n_{k, 1}$ and $n_{k, 0}$. Consider the treated samples first; we randomly partition $D_{k, 1}$ into a training set $D^\text{train}_{k, 1}$ of units $i \in \{1, ..., n_{k, 1}^\text{train}\}$ and a validation set $D_{k, 1}^\text{val}$ of units $i \in \{n_{k, 1}^\text{train} + 1, ..., n_{k, 1}\}$. Then, each candidate outcome model is fit on $D^\text{train}_{k, 1}$ to obtain $\widehat{m}_{a,n_{k,1}^{\text{train}}}^l$ for $l \in \mathcal{L}$. 

If the outcome is binary, the ensemble weights are determined by the fitted models' predictive risks evaluated on $D^\text{val}_{k, 1}$ according to the Bernoulli likelihood as shown in \eqref{equation_mix_propensity}. If the outcome is continuous, the ensemble weights are alternatively determined by the mean squared errors of the fitted models on $D^\text{val}_{k, 1}$ \cite{yang2004ensemble, Gu2019aggregate}. Specifically,  
\begin{equation}
    \begin{aligned}[b]
    \widehat{\Omega}_l &= \left(n_{k, 1}-{n_{k, 1}^\text{train}}\right)^{-1} \sum_{i=n_{k,1}^\text{train}+1}^{n_{k, 1}} \widehat{\Omega}_{l, i} \text{ and } \\
    \widehat{\Omega}_{l, i} &= \frac{\exp \left[-\kappa \sum_{q=n_{k,1}^{\text{train}}+1}^{i-1}\left\{Y_q-\widehat{m}_{n_{k,1}^{\text{train}}}^l\left(X_q\right)\right\}^2\right]}{\sum_{l^{\prime}=1}^{L} \exp \left[-\kappa \sum_{q=n_{k,1}^\text{train}+1}^{i-1}\left\{Y_q-\widehat{m}_{n_{k,1}^{\text{train}}}^{l^{\prime}}\left(X_q\right)\right\}^2\right]} \quad \text{ for } n_{k,1}^\text{train}+2 \leq i \leq n_{k, 1}, 
    \end{aligned}
    \label{equation_mix_outcome}
\end{equation}
where $\widehat{\Omega}_{l, n_{k, 1}^\text{train}+1} = 1 / L$ and the ensemble predictions for the conditional outcomes under treatment is $\widehat{m}_{1,k}(X_i) = \sum_{l =1}^{L} \widehat{\Omega}_l \widehat{m}_{1,k}^l(X_i)$. The above procedure is then repeated in the control group $D_{k, 0}$ and $a = 0$ to obtain the ensemble predictions for the conditional outcomes under control.

The tuning parameter $\kappa$ in \eqref{equation_mix_outcome} can be selected via cross-validation and \cite{Gu2019aggregate} showed that the performance of model mixture estimators is generally robust across different choices of $\kappa$; they recommended choosing $\kappa = \max (1,\lfloor\log (L)\rfloor)$. 

\section{Additional experimental details}

We consider five different estimators: (i) an augmented inverse probability weighted (AIPW) estimator using data from the target site only (Target), (ii) an AIPW estimator that weights each site proportionally to its sample size (SS), (iii) an AIPW estimator that weights each site inverse-proportionally to its variance (IVW), (iv) an AIPW estimator that weights each site with the $L_1$ weights defined in \eqref{equation_l1_weights} (AIPW-$L_1$), and (v) a multiply robust estimator with the $L_1$ weights defined in \eqref{equation_l1_weights} (MR-$L_1$). 

\subsection{Distribution of covariates}
We consider a total of five sites and fix the first site as the target site with a relatively small sample size of $300$. The source sites have larger sample sizes of $\left\{500, 500, 1000, 1000 \right\}$. We model heterogeneity in the covariate distributions across sites with skewed normal distributions and varying levels of skewness in each site, $X_{kp} \sim \mathcal{S N}\left(x ; \Xi_{k p}, \Omega_{k p}^2, \mathrm{~A}_{k p}\right)$, where $k \in \{1, ..., 5\}$ indexes each site and $p \in \{1, ..., 4\}$ indexes the covariates. Specifically, for any given site $k \in \mathcal{K}$, we set the location parameters $\Xi_{k p} = 0$ and the scale parameter $\Omega_{k p} = 1$. Moreover, in the case of the target site, we specifically assign the skewness parameter $\mathrm{~A}_{k p}$ a value of zero, denoting a symmetrical distribution. However, for source site $k \in \mathcal{S}$, we adopt different values for $\mathrm{~A}_{k p}$ based on the sample size. If the sample size is equal to $500$, we assign $\mathrm{~A}_{k p}$ the value of $(1/2)^p$, reflecting a positively skewed distribution. On the other hand, if the sample size is $1000$, we assign $\mathrm{~A}_{k p}$ the value of $-(1/2)^p$, indicating a negatively skewed distribution. Following \cite{kangschafer2007data-generation}, we generate covariates $Z_{kp}$ as non-linear transformation of $X_{kp}$ such that $Z_{k1} = \exp (X_{k1} / 2)$, $Z_{k2} = X_{k2} / \{1+\exp (X_{k1}) \}+10$, $Z_{k3}=(X_{k1} X_{k3} / 25+0.6)^3$ and $Z_{k4}=(X_{k2}+X_{k4}+20)^2$.

\subsection{Site-specific estimates}

The main text presented simulation results in Table \ref{table_varying_prop_correct_sites} based on $500$ simulations, assuming no covariate mismatch. Here, to provide further insights into the site-specific estimates, we provide simulation results in Figure \ref{figure_site_specific_simulation}. The results show that in cases where some sites fail to properly specify their working models, AIPW estimators exhibit significant bias, while the multiply robust estimators are capable of accurately recovering the true TATE. This can be attributed to the fact that the additional candidate model closely approximates the true underlying models. These findings underscore the enhanced safeguard against model misspecification provided by the multiply robust estimators.

\begin{figure}[!htbp]
     \centering
     \begin{subfigure}[h]{\textwidth}
         \centering
         \includegraphics[width=0.85\textwidth]{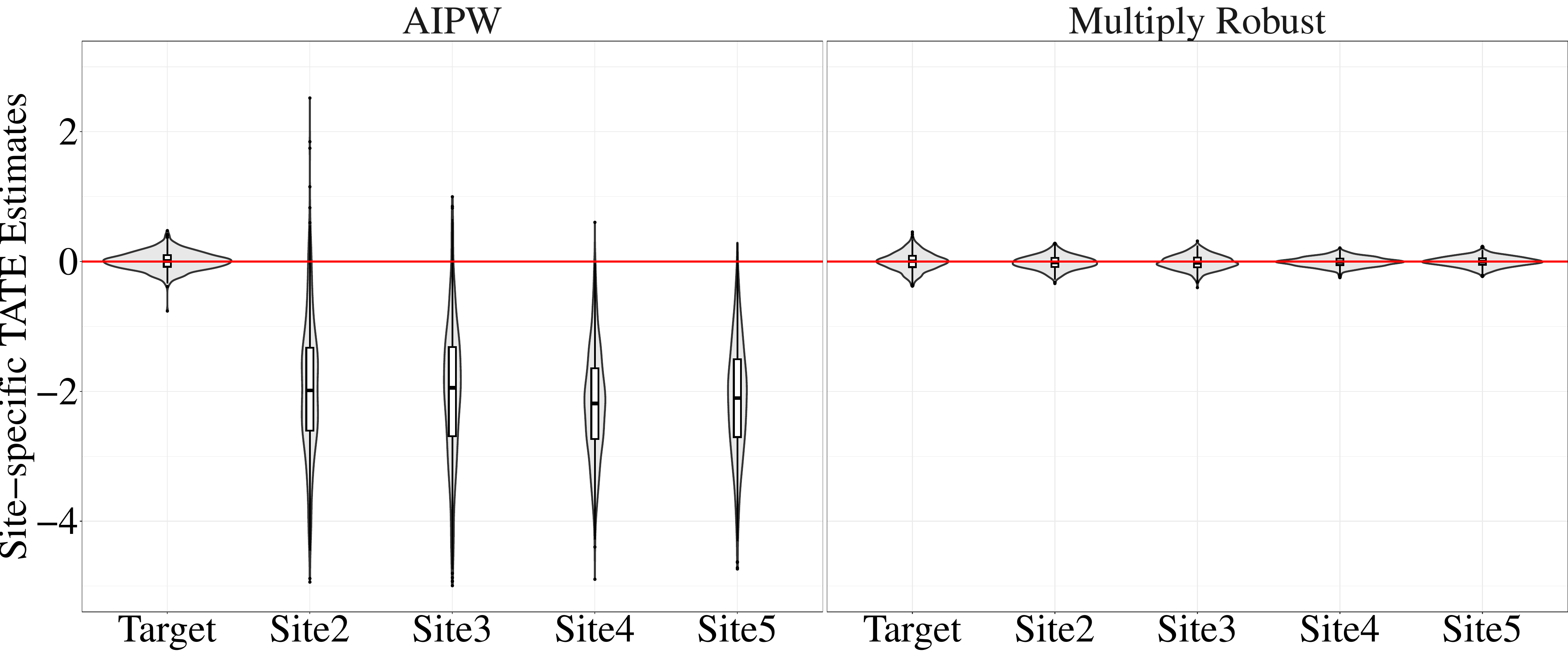}
         \caption{No Site Correctly Specify Models ($C = 0$)}
         \label{figure_site_specific_c0}
     \end{subfigure}
     \vspace{0.5em}
     \vfill
     \begin{subfigure}[h]{\textwidth}
         \centering
         \includegraphics[width=0.85\textwidth]{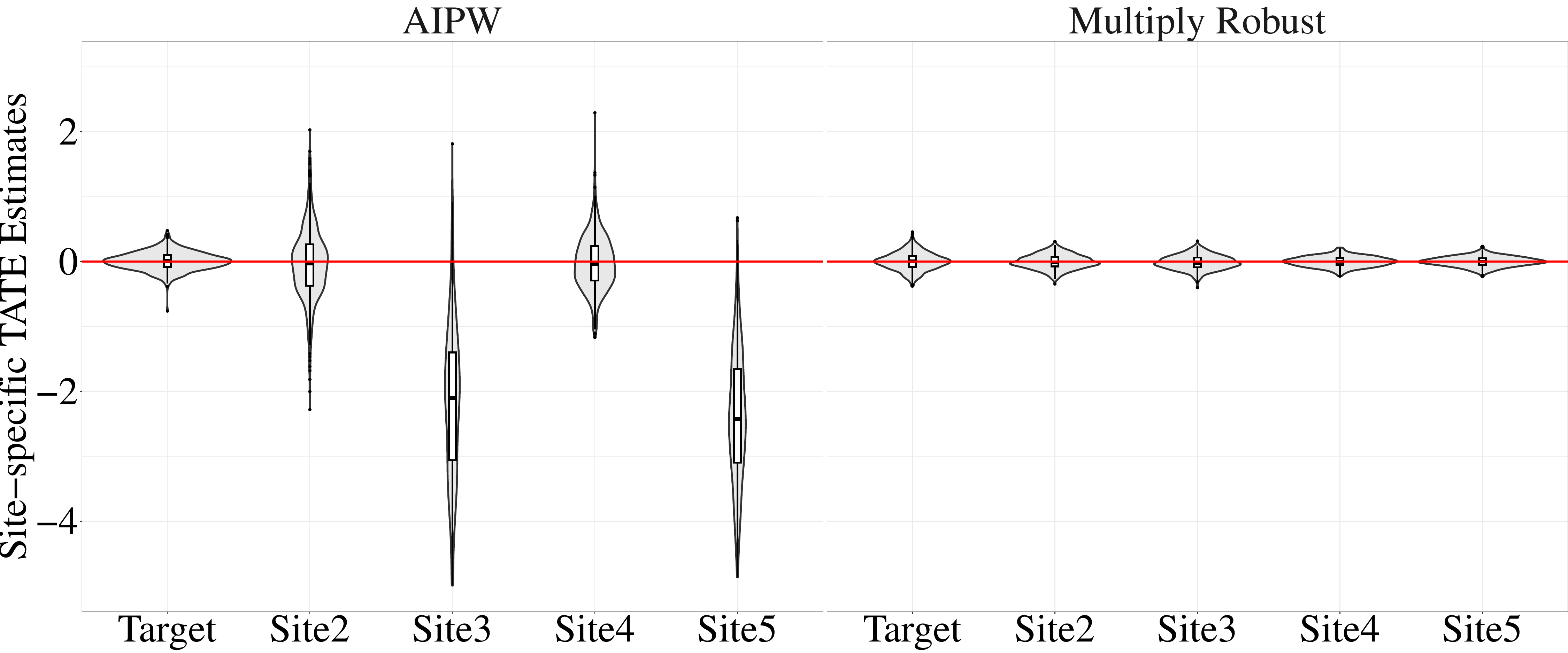}
         \caption{Site 2 and Site 4 Correctly Specify Models ($C = 1/2$)}
         \label{figure_site_specific_c0.5}
     \end{subfigure}
     \vspace{0.5em}
     \vfill
     \begin{subfigure}[h]{\textwidth}
         \centering
         \includegraphics[width=0.85\textwidth]{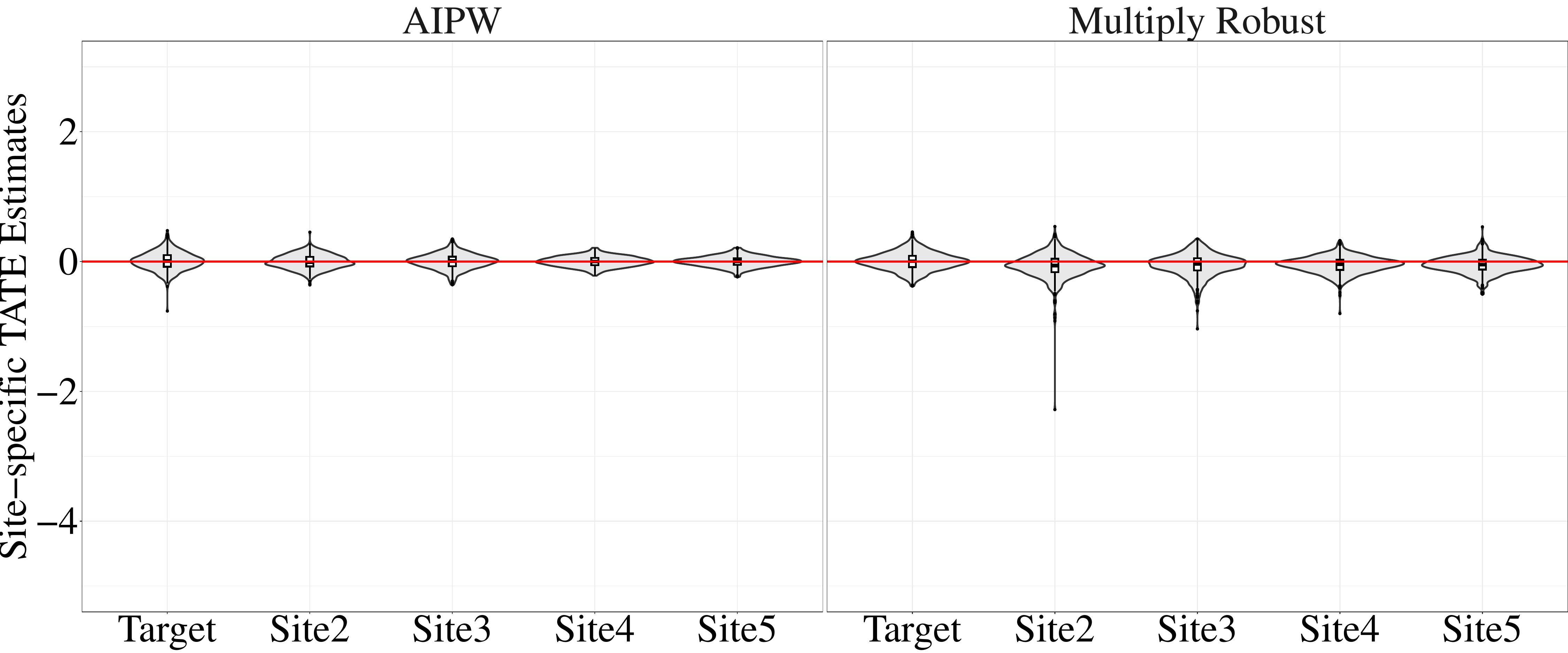}
         \caption{All Sites Correctly Specify Models ($C = 1$)}
         \label{figure_site_specific_c1}
     \end{subfigure}
        \caption{Results of site-specific estimation based on $500$ simulated data sets in three (mis)specification settings.}
        \label{figure_site_specific_simulation}
\end{figure}

\subsection{Reproducibility and computing resources}

All experiments in this study were performed using the statistical programming language \texttt{R} (version 4.2.2). The package \texttt{sn} (version 2.1.0) was employed for generating covariates that follow a skewed-normal distribution within each site. To solve the density ratio estimating equations, we utilized the \texttt{rootSolve} package (version 1.8.2.3). For the estimation of adaptive ensemble weights based on penalized regression of site-specific influence functions, we employed the \texttt{glmnet} package (version 4.1-4).

To enhance computational efficiency, parallel computing packages were employed. Specifically, the packages \texttt{foreach} (version 1.5.2) and \texttt{doParallel} were employed to facilitate the replication of experiments. For the purpose of model mixing in multiply robust estimation, we employed the \texttt{parallel} package (version 4.2.2). The replication of experiments was carried out using ten CPU cores, while the implementation of the model mixing algorithm utilized five CPU cores.

\section{Proofs}
\subsection{Proof of Theorem \ref{theorem_identification}}
    If Assumptions \ref{assumption_consistency} - \ref{assumption_target_positivity} hold, the mean counterfactual outcomes in the target population can be identified using data from the target site;  
    \begin{equation}
    \begin{aligned}[b]
        \mu_{a, T} & = E\left\{ Y_i{(a)} \mid R_i = T\right\} \\
        & = E [ E\left\{ Y_i{(a)} \mid V_{i} = v, R_i = T \right\} \mid R_i = T ] \\ 
        & = E [ E\left\{ Y_i{(a)} \mid V_{i} = v, A_i = a, R_i = T \right\} \mid R_i = T ] \\ 
        & = E [ E\left\{ Y_i \mid V_{i} = v, A_i = a, R_i = T \right\} \mid R_i = T ]. 
    \end{aligned} 
    \end{equation}
The second line follows the law of total expectation; the third line follows Assumption \ref{assumption_target_unconfoundedness}; the last line follows Assumption \ref{assumption_consistency}. \\
    
If Assumptions \ref{assumption_consistency}, \ref{assumption_source_unconfoundedness} - \ref{assumption_selection_positivity} hold, the mean counterfactual outcomes in the target population can be identified using data from source sites; 
    \begin{equation}
    \begin{aligned}[b]
        \mu_{a, T} & = E\left\{ Y_i{(a)} \mid R_i = T \right\} \\
        & = E [ E\left\{ Y_i{(a)} \mid V_{i} = v, R_i = T \right\} \mid R_i = T ] \\ 
        & = E [ E\left\{ Y_i{(a)} \mid V_{i} = v, R_i = k \right\} \mid R_i = T ] \\ 
        & = E\left( E [ E\left\{ Y_i{(a)}  \mid {X}_{i} = x, V_{i} = v, R_i = k \right\} \mid V_{i} = v, R_i = k ] \mid R_i = T \right) \\ 
        & = E\left( E [ E\left\{ Y_i{(a)}  \mid {X}_{i} = x, R_i = k \right\} \mid V_{i} = v, R_i = k ] \mid R_i = T \right) \\ 
        & = E\left( E [ E\left\{ Y_i(a)  \mid {X}_{i} = x, R_i = k, A_i = a \right\} \mid V_{i} = v, R_i = k ] \mid R_i = T \right) \\
        & = E\left( E [ E\left\{ Y_i  \mid {X}_{i} = x, R_i = k, A_i = a \right\} \mid V_{i} = v, R_i = k ] \mid R_i = T \right). 
    \end{aligned} 
    \end{equation}

The second line follows the law of total expectation; the third line follows Assumption \ref{assumption_selection_unconfoundedness}; the fourth line follows the law of total expectation; the fifth line follows by our setup that $V_i \subseteq X_i$; the sixth line follows Assumption \ref{assumption_source_unconfoundedness} the last line follows Assumption \ref{assumption_consistency}.

\subsection{Lemmas in \cite{li2020demystifying}}
\label{appendix_li_lemmas}
% emphasize in the main text that these two lemmas are from Li et al.  
We restate Lemmas \ref{lemma1_mix_propensity_risk} and \ref{lemma2_mix_outcome_risk} that were proved in \cite{li2020demystifying}. These two lemmas together show that the $L_2$ risks of the multiply robust estimators for $\pi_{a,k}$ and $m_{a,k}$ are bounded by the $L_2$ risks of the model with the smallest risks, plus a negligible remainder term. 
\begin{lemma}
\label{lemma1_mix_propensity_risk}
Suppose that for each $j \in \mathcal{J}$, there exists a constant $0 < \epsilon_j < 1/2$ such that $\epsilon_j < \widehat{\pi}_{a, k}^j(x) < 1-\epsilon_j$ for all $x$. Then 
\begin{equation}
    E\left(\left\| \widehat{\pi}_{a, k} -\pi_{a, k} \right\|^2\right) \leq \inf _{j \in \mathcal{J}} \frac{2}{\epsilon_j^2} E\left(\|\widehat{\pi}_{a, k}^j-\pi_{a, k} \|^2\right)+\frac{2 \log (J)}{n_k - n_k^{\text{train}}}
\end{equation}
where $\widehat{\pi}_{a, k}(x) = \sum_{j =1}^{J} \widehat{\Lambda}_j \widehat{\pi}^j_{a, k}(x)$. 
\end{lemma}

If the candidate propensity models are parametric and one of them is correctly specified for $\pi_{a, k}$, then $\inf _{j \in \mathcal{J}} \frac{2}{\epsilon_j^2} E( \|\widehat{\pi}_{a, k}^j-\pi_{a, k} \|^2 )$ converges at a rate of $1/n$. If the candidate propensity models are non-parametric, then $\inf _{j \in \mathcal{J}} \frac{2}{\epsilon_j^2} E( \|\widehat{\pi}_{a, k}^j-\pi_{a, k} \|^2 )$ converges at a rate slower than $1/n$. The remainder term $\frac{2 \log (J)}{n_k - n_k^{\text{train}}}$ converges at the rate $1 / n_k - n_k^{\text{train}}$, so it vanishes at a faster rate than the statistical risks of candidate models themselves. 

\begin{lemma}
\label{lemma2_mix_outcome_risk}
Suppose we have a continuous outcome and follow the continuous outcome model mixing algorithm presented in Appendix \ref{appendix_multiply_robust_outcome}. Suppose there exist constants $C_1, C_2 > 0$ such that $\sup_{l \in \mathcal{L}} | \hat{m}_{a, k}^l(x) - m_{a, k}(x) | \leq C_1$ for all $x$ and the subexponential norm of $Y - m_{a, k}(x)$ given $X = x$ is bounded above by $C_2$ for all $x$. Then for $a \in \{0, 1\}$
\begin{equation*}
    E \left( \left\| \widehat{m}_{a, k} - m_{a, k}  \right\|^2\right) \leq \inf _{l \in \mathcal{L}} E \left( \| \widehat{m}_{a, k}^l - m_{a, k} \|^2 \right) + \frac{ \log (L)}{\kappa \left( n_{k, a} - n_{k, a}^{\text{train}} \right)}
\end{equation*}
for 
\begin{equation}
    0<\kappa \leq \max \left[\frac{1}{16  { e } C_1 C_2}, \frac{\exp \left\{C_1\left(8  { e } C_2\right)^{-1}\right\}}{4 \mathcal{M}_2\left\{\left(4  { e } C_2\right)^{-1}\right\}+16 C_1^2 \mathcal{M}_0\left\{\left(4 { e } C_2\right)^{-1}\right\}}\right]
\end{equation}
where $\mathcal{M}_0(t)=2 \exp \left(2 {e}^2 C_2^2 t^2\right), \mathcal{M}_2(t)=16 \sqrt{ } 2 C_2^2 \exp \left(8 {e}^4 C_2^2 t^2\right) \text { and } {e}=\exp (1)$.
\end{lemma}

\subsection{Influence function of site-specific estimator}
\label{appendix_site_estimator_influence}
In this section, we derive the general form of the influence functions for the site-specific estimators, following a similar derivation in \cite{zeng2023efficient}. The primary difference is that we use the density ratio weights $\zeta_{k}(V)$ instead of the inverse probability of selection weights. Define $Z_i = (Y_i, A_i, X_i, R_i)$ with the (partial) baseline covariates in the target site as $V_i \subseteq X_i$. The general form for our efficient influence function is 
\begin{equation}
    \begin{aligned}[b]
    \xi_{a, k}(Z_i) & =\frac{1}{P(R_i=k)} \Biggl[ \frac{I(A_i=a, R_i=k)}{ \pi_{a, k}(X_i)} \zeta_{k}(V_i) \{ Y_i -m_{a, k}(X_i) \} \Biggr] \\
    & + \frac{1}{P(R_i =k)} 
    \Biggl[ I(R_i=k) \zeta_{k}(V_i) \{ m_{a,k}(X_i)-\tau_{a, k}(V_i) \} \Biggr] \\ 
    & + \frac{1}{P(R_i=T)} \Biggl[ I(R_i=T)  \tau_{a, k}(V_i) \Biggr] - \mu_{a, T}.
    \end{aligned}
    \label{equation_influence_function}
\end{equation}

%We use $V_i$ for the treatment and outcome models for the target site estimator. We do not account for covariate mismatch or covariate shift. 
The target site estimator has the following form, 
\begin{equation}
    \widehat{\mu}_{a, T} = \frac{1}{n_T} \sum_{i = 1}^{n} \Biggl[ \frac{I(A_i=a, R_i=T)}{\pi_{a, T}(V_{i})} \left\{ Y_i- m_{a, T}(V_{i}) \right\} + I(R_i=T) m_{a, T}(V_{i}) \Biggr]. 
\end{equation}
This is a standard AIPW estimator whose influence function is derived in \cite{dahabreh2019generalizing, han2022privacy} as 

\begin{equation}
    \xi_{a, T}(Z_i) =\frac{1}{P(R_i=T)} \Biggl[ \frac{I(A_i=a, R_i=T)}{ \pi_{a, T}(V_i)} \{ Y_i-m_{a, T}(V_i) \} + I(R_i=T)  m_{a, T}(V_i) \Biggr] - \mu_{a, T}.
\end{equation}

\subsection{Proof of Theorem \ref{theorem_site_consistency_normality}}

We first prove that given the conditions in Theorem \ref{theorem_site_consistency_normality}, the site-specific estimator is a consistent estimator of $\mu_{a, T}$, formalized in the following Lemma: 
\begin{lemma}
    \label{lemma_site_consistency}
    Given Assumptions \ref{assumption_consistency} - \ref{assumption_selection_positivity}, and \ref{lemma1_mix_propensity_risk} and \ref{lemma2_mix_outcome_risk}, 
    if one of \ref{assumption_multiply_robust_B1} or \ref{assumption_multiply_robust_B2} and one of \ref{assumption_multiply_robust_C1} or \ref{assumption_multiply_robust_C2} hold, then $\widehat{\mu}_{a, k}$ is a consistent estimator for $\mu_{a, T}$. 
\end{lemma}

\begin{proof}[Proof of Lemma \ref{lemma_site_consistency}]
We divide the proof into four cases and show that $\widehat{\mu}_{a, k}$ achieves consistency. 
\begin{enumerate}[label=Case \arabic*: ]
    \item When $\overline{\pi}_{a, k}^j = \pi_{a, k}$ and $\overline{\tau}_{a, k} = \tau_{a, k}$ \\ 
    By assumption, we have $\| \widehat{\pi}_{a, k}^j - {\pi}_{a, k} \| = o_p(1)$ and $\| \widehat{\tau}_{a, k} - {\tau}_{a, k} \| = o_p(1)$. Given Lemma \ref{lemma1_mix_propensity_risk} for $\widehat{\pi}_{a,k}$, $\| \widehat{\pi}_{a, k} - {\pi}_{a, k} \| = o_p(1)$. Define $\overline{m}_{a, k}(x) = \frac{1}{L}\sum_{l = 1}^L \overline{m}_{a, k}^l(x)$. By definition, $\| \widehat{m}_{a, k}^l - \overline{m}_{a, k}^l \| = o_p(1)$, so $\| \widehat{m}_{a, k} - \overline{m}_{a, k} \| = o_p(1)$. Together with $\| \widehat{\zeta}_{k} - \overline{\zeta}_{k} \| = o_p(1)$, we can re-write the site-specific estimator as 
    \begin{equation}
    \begin{aligned}[b]
    \widehat{\mu}_{a, k} = & \frac{1}{n} \sum_{i = 1}^{n} \biggl[ \frac{n}{n_k} \frac{I(A_i=a, R_i=k)}{\pi_{a, k}(X_{i})} {\overline{\zeta}}_k (V_i) \Bigl\{Y_i-\overline{m}_{a, k}(X_{i})\Bigr\} \biggr] \\
    & + \frac{1}{n} \sum_{i = 1}^{n} \biggl[ \frac{n}{n_k} I(R_i=k) {\overline{\zeta}}_k (V_i) \Bigl\{ \overline{m}_{a, k}(X_{i})-\tau_{a, k}(V_i) \Bigr\} \biggr] \\
    & + \frac{1}{n} \sum_{i = 1}^{n} \frac{n}{n_T} I(R_i=T) \tau_{a, k}(V_i) \\
    & + o_p(1) 
    \end{aligned}
    \end{equation}
    By assumption of i.i.d. units within each site and the law of large numbers, $\widehat{\mu}_{a, k}$ converges in probability to
    {\allowdisplaybreaks
    \begin{equation}
    \begin{aligned}[b]
    & E \biggl[ \frac{n}{n_k} \frac{I(A_i=a, R_i=k)}{\pi_{a, k}(X_{i})} {\overline{\zeta}}_k (V_i) \Bigl\{Y_i-\overline{m}_{a, k}(X_{i})\Bigr\} \biggr] \\
    + & E \biggl[ \frac{n}{n_k} I(R_i=k) {\overline{\zeta}}_k (V_i) \Bigl\{ \overline{m}_{a, k}(X_{i})-\tau_{a, k}(V_i) \Bigr\} \biggr] \\
    + & E \biggl[ \frac{n}{n_T} I(R_i=T) \tau_{a, k}(V_i) \biggr] \\
    = & \underbrace{ E \biggl[ \frac{n}{n_k} I(R_i=k) {\overline{\zeta}}_k (V_i) \Bigl\{ \frac{I(A_i=a)}{\pi_{a, k}(X_{i})} Y_i - \tau_{a, k}(V_i) \Bigr\} \biggr] }_{T_1} \\
    - & \underbrace{ E \biggl[ \frac{n}{n_k} I(R_i=k) {\overline{\zeta}}_k (V_i) \Bigl\{ \frac{I(A_i=a)}{\pi_{a, k}(X_{i})} - 1 \Bigr\} \overline{m}_{a, k}(X_{i}) \biggr] }_{T_2} \\
    + & \underbrace{ E \biggl[ \frac{n}{n_T} I(R_i=T) \tau_{a, k}(V_i) \biggr] }_{T_3}. 
    \end{aligned}
    \end{equation}}
    
    Given that $\pi_{a, k}$ is the true propensity score model, $E(T_2) = 0$ since
    \begin{equation}
        E \Bigl\{ \frac{ I(A_i=a) }{ \pi_{a, k}(X_{i})} - 1 \mid X_i, R_i = k \Bigr\} = \frac{ P(A_i = a \mid X_i, R_i = k)}{\pi_{a, k}(X_i)} - 1 = 0. 
    \end{equation}
    Similarly, given that $\tau_{a, k} = E\{ m_{a,k}(X_i) \mid V_i, R_i = k \}$ is the true model, 
    \begin{equation}
        \begin{aligned}[b]
        E(T_3) & = E \{ E(T_3 \mid R_i = T) \} \\
        & = \frac{n}{n_T} P(R_i = T) E( \tau_{a, k}(V_i) \mid R_i = T) \\
        & = E( \tau_{a, k}(V_i) \mid R_i = T)\\
        & = \mu_{a, T}. 
        \end{aligned}
    \end{equation}
    Finally, we consider $T_1$; 
    {\allowdisplaybreaks
    \begin{align}
        E(T_1) 
        &= E \biggl[ \frac{n}{n_k} I(R_i=k) {\overline{\zeta}}_k (V_i) \Bigl\{ \frac{I(A_i=a)}{\pi_{a, k}(X_{i})} Y_i - \tau_{a, k}(V_i) \Bigr\} \biggr] \nonumber \\
        &= E \Biggl( E \Bigl[ \frac{n}{n_k} I(R_i=k)  \overline{\zeta}_k (V_i) \Bigl\{\frac{I(A_i=a)}{\pi_{a, k}(X_{i})} Y_i-\tau_{a, k}(V_i)\Bigr\} \mid V_i \Bigr] \Biggr) \nonumber \\
        &= E \Biggl( \frac{n}{n_k} I(R_i=k) {\overline{\zeta}}_k (V_i) \Bigl[ E \Bigl\{ \frac{I(A_i=a)}{\pi_{a, k}(X_{i})} Y_i \mid V_i \Bigr\} -\tau_{a, k}(V_i) \Bigr] \Biggr) \nonumber \\
        &= E \Biggl( \frac{n}{n_k} I(R_i=k) {\overline{\zeta}}_k (V_i) \Bigl[ E \Bigl\{ \frac{I(A_i=a)}{\pi_{a, k}(X_{i})} Y_i \mid V_i, R_i = k \Bigr\} -\tau_{a, k}(V_i) \Bigr] \Biggr) \nonumber \\
        &= E \Biggl\{ \frac{n}{n_k} I(R_i=k) {\overline{\zeta}}_k (V_i) \Biggl( E \Bigl[ E \Bigl\{ \frac{I(A_i=a)}{\pi_{a, k}(X_{i})} Y_i \mid X_i, R_i = k \Bigr\} \mid V_i, R_i = k \Bigl] -\tau_{a, k}(V_i) \Biggr) \Biggr\} \nonumber \\
        &= E \Biggl( \frac{n}{n_k} I(R_i=k) {\overline{\zeta}}_k (V_i) \Bigl[ E \Bigl\{ E ( Y_i \mid A_i = a, X_i, R_i = k ) \mid V_i, R_i = k \Bigl\} -\tau_{a, k}(V_i) \Bigr] \Biggr) \nonumber \\
        &= E \Biggl( \frac{n}{n_k} I(R_i=k) {\overline{\zeta}}_k (V_i) \Bigl[ E \Bigl\{ m_{a, k}(X_i) \mid V_i, R_i = k \Bigl\} -\tau_{a, k}(V_i) \Bigr] \Biggr) \nonumber \\
        &= 0, \text{ which completes the proof for Case 1.}
    \end{align}}

    \item When $\overline{\pi}_{a, k}^j = \pi_{a, k}$ and $\overline{\zeta}_{k} = \zeta_{k}$ \\
    By assumption, we have $\| \widehat{\pi}_{a, k}^j - {\pi}_{a, k} \| = o_p(1)$ and $\| \widehat{\zeta}_{k} - {\zeta}_{k} \| = o_p(1)$. Given Lemma \ref{lemma1_mix_propensity_risk} for $\widehat{\pi}_{a,k}$, $\| \widehat{\pi}_{a, k} - {\pi}_{a, k} \| = o_p(1)$. Define $\overline{m}_{a, k}(x) = \frac{1}{L}\sum_{l = 1}^L \overline{m}_{a, k}^l(x)$. By definition, $\| \widehat{m}_{a, k}^l - \overline{m}_{a, k}^l \| = o_p(1)$, so $\| \widehat{m}_{a, k} - \overline{m}_{a, k} \| = o_p(1)$. Together with $\| \widehat{\tau}_{a, k} - \overline{\tau}_{a, k} \| = o_p(1)$, we can write the site-specific estimator as 
    {\allowdisplaybreaks
    \begin{align}
    \widehat{\mu}_{a, k} = & \frac{1}{n} \sum_{i = 1}^{n} \biggl[ \frac{n}{n_k} \frac{I(A_i=a, R_i=k)}{\pi_{a, k}(X_{i})} \zeta_k (V_i) \Bigl\{Y_i-\overline{m}_{a, k}(X_{i})\Bigr\} \biggr] \nonumber \\
    & + \frac{1}{n} \sum_{i = 1}^{n} \biggl[ \frac{n}{n_k} I(R_i=k) \zeta_k (V_i) \Bigl\{ \overline{m}_{a, k}(X_{i})-\overline{\tau}_{a, k}(V_i) \Bigr\} \biggr] \nonumber \\
    & + \frac{1}{n} \sum_{i = 1}^{n} \frac{n}{n_T} I(R_i=T) \overline{\tau}_{a, k}(V_i) \nonumber \\
    & + o_p(1) 
    \end{align}}
    By assumption of i.i.d. units within each site and the law of large numbers, $\widehat{\mu}_{a, k}$ converges in probability to
    {\allowdisplaybreaks
    \begin{align}
    & E \biggl[ \frac{n}{n_k} \frac{I(A_i=a, R_i=k)}{\pi_{a, k}(X_{i})} \zeta_k (V_i) \Bigl\{Y_i-\overline{m}_{a, k}(X_{i})\Bigr\} \biggr] \nonumber \\
    + & E \biggl[ \frac{n}{n_k} I(R_i=k) \zeta_k (V_i) \Bigl\{ \overline{m}_{a, k}(X_{i})-\overline{\tau}_{a, k}(V_i) \Bigr\} \biggr] \nonumber \\
    + & E \biggl[ \frac{n}{n_T} I(R_i=T) \overline{\tau}_{a, k}(V_i) \biggr] \nonumber \\
    = & \underbrace{ E \Bigl\{ \frac{n}{n_k} \frac{I(A_i=a, R_i=k)}{\pi_{a, k}(X_{i})} \zeta_k (V_i) Y_i \Bigr\} }_{T_1} \nonumber \\
    + & \underbrace{ E \biggl[ \frac{n}{n_k} I(R_i=k) \zeta_k (V_i) \overline{m}_{a, k}(X_{i}) \Bigl\{ 1 - \frac{I(A_i=a)}{\pi_{a, k}(X_{i})} \Bigr\} \biggr] }_{T_2} \nonumber \\
    + & \underbrace{ E \biggl[ \overline{\tau}_{a, k}(V_i) \Bigl\{ \frac{n}{n_T} I(R_i=T) - \frac{n}{n_k} I(R_i=k) \zeta_k({V_i}) \Bigr\} \biggr] }_{T_3} 
    \end{align}}

    Given that $\pi_{a, k}$ is the true propensity score model, $E(T_2) = 0$ since
    \begin{equation}
        E \Bigl\{ \frac{ I(A_i=a) }{ \pi_{a, k}(X_{i})} - 1 \mid X_i, R_i = k \Bigr\} = \frac{ P(A_i = a \mid X_i, R_i = k)}{\pi_{a, k}(X_i)} - 1 = 0. 
    \end{equation}

    Also since $\zeta_{k}= {f(V_i | R_i = T)} / {f(V_i | R_i = k)}$ is the true density ratio model, 
    \begin{equation}
        \begin{aligned}[b]
        E(T_3) 
        & = E \biggl( E \biggl[ \overline{\tau}_{a, k}(V_i) \Bigl\{ \frac{n}{n_T} I(R_i=T) - \frac{n}{n_k} I(R_i=k) \zeta_k({V_i}) \Bigr\} \mid V_i \biggr] \biggr)  \\ 
        & = E \biggl( \overline{\tau}_{a, k}(V_i) \biggl[ \frac{n}{n_T} E \Bigl\{ I(R_i=T) \Bigr\} - \frac{n}{n_k} E \Bigl\{ I(R_i=k) \Bigr\} \zeta_k({V_i}) \mid V_i \biggr] \biggr)  \\ 
        & = 0 \text{  following the results from Appendix \ref{appendix_equivalence_density_ratio}. }
        \end{aligned}
    \end{equation}
    Thus, we are left with $T_1$, which equals to $\mu_{a, T}$ by iterative conditioning and Assumptions \ref{assumption_consistency}, \ref{assumption_target_unconfoundedness}, \ref{assumption_source_unconfoundedness} and \ref{assumption_selection_unconfoundedness}. 
\end{enumerate}
Proofs for Case 3, when $\overline{m}_{a, k}^l = m_{a, k}$ and $\overline{\tau}_{a, k} = \tau_{a, k}$ and Case 4, when $\overline{m}_{a, k}^l = m_{a, k}$ and $\overline{\zeta}_{k} = \zeta_{k}$, follow closely to the proofs of Cases 1 and 2 and are omitted here for brevity. 
\end{proof}

\begin{proof}[Proof of Theorem \ref{theorem_site_consistency_normality}]
Given the general form of influence function \eqref{equation_influence_function}, we can decompose the estimation risk of the site-specific estimator as 
\begin{equation}
    \mathbb{P}_n (\widehat{\mu}_{a, k}) -{E} (\mu_{a, T}) =\left(\mathbb{P}_n-{E}\right)(\widehat{\mu}_{a, k}-\mu_{a, T})+\left(\mathbb{P}_n-{E}\right)\mu_{a, T}+{E}(\widehat{\mu}_{a, k}-\mu_{a, T})
\end{equation}

Assuming proper sample splitting methods are employed, as stated in Lemma 2 of \cite{kennedy2022minimax}, we have: 
\begin{equation}
    \left(\mathbb{P}_n - E\right)\left(\widehat{\mu}_{a, k}-\mu_{a, T}\right)=o_p(n^{-1/2}).
\end{equation}

Then, we consider $\left(\mathbb{P}_n - E\right)\mu_{a, T}$. By the Central Limit Theorem, we have $\sqrt{n}\left\{ \mathbb{P}_n(\mu_{a, T}) - E(\mu_{a, T}) \right\} \stackrel{d}{\rightarrow} \mathcal{N}(0, \sigma^2)$, where $\sigma^2$ is given by:
{\allowdisplaybreaks
\begin{align}
    \sigma^2 
    & = E\left\{\xi_{a, k}(Z_i)^2 \right\} \nonumber \\
    & = E \left\{ \frac{1}{P^2(R_i=k)} \Biggl[ \frac{I(A_i=a, R_i=k)}{ \pi_{a, k}(X_i)} \zeta_{k}(V_i) \{ Y_i -m_{a, k}(X_i) \} \Biggr]^2 \right\} \nonumber \\
    & + E \left\{ \frac{1}{P^2(R_i =k)}
    \Biggl[ I(R_i=k) \zeta_{k}(V_i) \{ m_{a,k}(X_i)-\tau_{a, k}(V_i) \} \Biggr]^2 \right\} \nonumber \\ 
    & + E \left\{ \frac{1}{P^2(R_i=T)} \Biggl[ I(R_i=T)  \left\{ \tau_{a, k}(V_i) - \mu_{a, T} \right\} \Biggr]^2 \right\} \nonumber \\ 
    & + \text{Remaining cross-terms}. 
\end{align}}
It can be verified that all remaining cross-terms have an expected value of zero. Hence,
{\allowdisplaybreaks
\begin{align}
    \sigma^2 
    & = E \left\{ \frac{1}{P^2(R_i=k)} \left[ \frac{I(A_i=a, R_i=k)}{ \pi^2_{a, k}(X_i)} \zeta^2_{k}(V_i) \{ Y_i -m_{a, k}(X_i) \}^2  \right] \right\} \nonumber \\
    & + E \left\{ \frac{1}{P^2(R_i =k)} 
     I(R_i=k) \left[ \zeta_{k}(V_i) \{ m_{a,k}(X_i)-\tau_{a, k}(V_i) \} \right]^2 \right\} \nonumber \\ 
    & + E \left[ \frac{1}{P^2(R_i=T)} I(R_i=T) \left\{ \tau_{a, k}(V_i) - \mu_{a, T} \right\}^2 \right] \nonumber \\ 
    & = E \left[ E \left\{ \frac{1}{P^2(R_i=k)} \left[ \frac{I(A_i=a, R_i=k)}{ \pi^2_{a, k}(X_i)} \zeta^2_{k}(V_i) \{ Y_i -m_{a, k}(X_i) \}^2 \right] \mid X_i \right\} \right] \nonumber \\
    & + E \left[ E \left\{ \frac{1}{P^2(R_i =k)} 
     I(R_i=k) \zeta^2_{k}(V_i) \{ m_{a,k}(X_i)-\tau_{a, k}(V_i) \}^2 \mid V_i \right\} \right] \nonumber \\ 
    & + E \left[ \frac{1}{P^2(R_i=T)} I(R_i=T) \left\{ \tau_{a, k}(V_i) - \mu_{a, T} \right\}^2 \right] \nonumber \\ 
    & = E \left[ \frac{P(R_i=k \mid X_i )}{P^2(R_i=k)} \zeta^2_{k}(V_i) E \left\{ \frac{I(A_i=a)}{ \pi^2_{a, k}(X_i)} \{ Y_i -m_{a, k}(X_i) \}^2 \mid X_i, A_i = a, R_i = k \right\} \right] \nonumber \\
    & + E \left\{ \frac{P(R_i=k \mid X_i)}{P^2(R_i =k)} \zeta^2_{k}(V_i) E \left[ 
    \{ m_{a,k}(X_i)-\tau_{a, k}(V_i) \}^2 \mid X_i, R_i = k \right] \right\} \nonumber \\ 
    & + E \left( \frac{P(R_i=T)}{P^2(R_i=T)} E \left[ \left\{ \tau_{a, k}(V_i) - \mu_{a, T} \right\}^2 \mid R_i = T \right] \right) \nonumber \\ 
    & = E \left[ \frac{P(R_i=k \mid V_i)}{P^2(R_i=k)} \zeta^2_{k}(V_i) \frac{\text{Var}\{ Y_i \mid X_i, A_i = a, R_i = k \}}{ \pi_{a, k}(X_i)} \right] \nonumber \\
    & + E \left[ \frac{P(R_i=k \mid V_i)}{P^2(R_i =k)} \zeta^2_{k}(V_i) \text{Var}
    \{ m_{a,k}(X_i) \mid V_i, R_i = k \} \right] \nonumber \\ 
    & + E \left[ \frac{\text{Var} \left\{ \tau_{a, k}(V_i) \mid R_i = T \right\}}{P(R_i=T)} \right]. \label{equation_semiparametric_eff_bound}
\end{align}
}

This expression represents the semiparametric efficiency bound for estimating $\mu_{a, T}$ and is finite given that all nuisance functions are uniformly bounded. Consequently, we can conclude that $(\mathbb{P}_n - E) (\mu_{a, T}) = O_p(n^{-1/2})$. 

Moving forward, we analyze the conditional bias term. Firstly, let us define $E\left\{\widehat{m}_{a, k}(X_{i}) \mid V_i, R_i = k \right\} = \tilde{\tau}_{a, k}(V_{i})$. Given conditions in Theorem \ref{theorem_site_consistency_normality} and Lemma \ref{lemma_site_consistency}, we can rewrite the conditional bias as follows:
{\allowdisplaybreaks
\begin{align}
    E \left( \widehat{\mu}_{a, k}-\mu_{a, T}  \right)
    & = E \left( \frac{n}{n_k} \biggl[ \frac{I(A_i=a, R_i=k)}{\widehat{\pi}_{a, k}(X_{i})} \widehat{\zeta}_{k}(V_{i}) \Bigl\{ m_{a, k}(X_i) -\widehat{m}_{a, k}(X_{i}) \Bigr\} \biggr] \right) \nonumber \\
    & + E \left( \frac{n}{n_k} \biggl[ I(R_i=k) \widehat{\zeta}_{k}(V_{i})
    \Bigl\{ \tilde{\tau}_{a, k}(V_{i}) -\widehat{\tau}_{a, k}(V_{i}) \Bigr\} \biggr] \right) \nonumber \nonumber \\
    & + E \left( \frac{n}{n_T} \biggl[ I(R_i = T) \Bigl\{\widehat{\tau}_{a, k}(V_{i}) - \mu_{a, T}(V_{i}) \Bigr\} \biggr] \right) \nonumber \\ 
    & = E \left( \frac{n}{n_k} \biggl[ \frac{I(A_i=a, R_i=k)}{\widehat{\pi}_{a, k}(X_{i})} \widehat{\zeta}_{k}(V_{i}) \Bigl\{ m_{a, k}(X_i) -\widehat{m}_{a, k}(X_{i}) \Bigr\} \biggr] \right) \nonumber \\
    & + E \left( \frac{n}{n_k} \biggl[ I(R_i=k) \widehat{\zeta}_{k}(V_{i})
    \Bigl\{ \tilde{\tau}_{a, k}(V_{i}) - \tau_{a, k}(V_{i}) \Bigr\} \biggr] \right) \nonumber \\
    & - E \left( \frac{n}{n_k} \biggl[ I(R_i=k) \widehat{\zeta}_{k}(V_{i})
    \Bigl\{ \widehat{\tau}_{a, k}(V_{i}) - {\tau}_{a, k}(V_{i}) \Bigr\} \biggr] \right) \nonumber \\
    & + E \left( \frac{n}{n_T} \biggl[ I(R_i = T) \Bigl\{ \widehat{\tau}_{a, k}(V_{i}) - \mu_{a, T}(V_{i})\Bigr\} \biggr] \right) \nonumber \\ 
    & = E \left( \frac{n}{n_k} \biggl[ \frac{I(A_i=a, R_i=k)}{\widehat{\pi}_{a, k}(X_{i})} \widehat{\zeta}_{k}(V_{i}) \Bigl\{ m_{a, k}(X_i) -\widehat{m}_{a, k}(X_{i}) \Bigr\} \biggr] \right) \nonumber \\
    & + E \left( \frac{n}{n_k} \biggl[ I(R_i=k) \widehat{\zeta}_{k}(V_{i})
    \Bigl\{ \widehat{m}_{a, k}(X_{i}) - m_{a, k}(X_{i}) \Bigr\} \biggr] \right) \nonumber \\
    & - E \left( \frac{n}{n_k} \biggl[ I(R_i=k) \widehat{\zeta}_{k}(V_{i})
    \Bigl\{ \widehat{\tau}_{a, k}(V_{i}) - {\tau}_{a, k}(V_{i}) \Bigr\} \biggr] \right) \nonumber \\
    & + E \left( \frac{n}{n_T} \biggl[ I(R_i = T) \Bigl\{ \widehat{\tau}_{a, k}(V_{i}) - \mu_{a, T}(V_{i}) \Bigr\} \biggr] \right) \nonumber \\ 
    & = \underbrace{E \left[ \frac{n}{n_k} \widehat{\zeta}_{k}(V_{i}) \Bigl\{ \frac{I(A_i=a, R_i=k)}{\widehat{\pi}_{a, k}(X_{i})}  - I(R_i=k) \Bigr\}  \Bigl\{ m_{a, k}(X_i) -\widehat{m}_{a, k}(X_{i}) \Bigr\} \right]}_{T_1} \nonumber \\
    & + \underbrace{E \left[ \Bigl\{ \frac{n}{n_T} I(R_i = T) - \frac{n}{n_k} I(R_i = k) \widehat{\zeta}_{k}(V_{i}) \Bigr\} \Bigl\{ \widehat{\tau}_{a, k}(V_{i}) - \tau_{a, T}(V_{i}) \Bigr\} \right]}_{T_2}. 
\end{align}
}

Step 3 is derived from Step 2 through conditioning on the variables $(V_i, R_i = k)$, and we obtain the following equality:
\begin{equation}
    E \left( \frac{n}{n_k} \biggl[ I(R_i=k) \widehat{\zeta}_{k}(V_{i})
    \Bigl\{ \tilde{\tau}_{a, k}(V_{i}) - \tau_{a, k}(V_{i}) \Bigr\} \biggr] \right) = E \left( \frac{n}{n_k} \biggl[ I(R_i=k) \widehat{\zeta}_{k}(V_{i})
    \Bigl\{ \widehat{m}_{a, k}(V_{i}) - m_{a, k}(V_{i}) \Bigr\} \biggr] \right).
\end{equation}
Step 4 can be derived from Step 3 by definition; in the target site, the conditional outcome is represented by $\mu_{a, T} = {m}_{a, T}(V_i)$ and $\tau_{a, T} = E\left\{ {m}_{a, T}(V_i) \mid V_i, R_i = T \right\} = {m}_{a, T}(V_i) = \mu_{a, T}$. Conditioning on $X_i$, we have the following expression:
\begin{align}
    T_1 = 
    & E \left[ \frac{n}{n_k} \widehat{\zeta}_{k}(V_{i}) \Bigl\{ \frac{I(A_i=a, R_i=k)}{\widehat{\pi}_{a, k}(X_{i})}  - I(R_i=k) \Bigr\}  \Bigl\{ m_{a, k}(X_i) -\widehat{m}_{a, k}(X_{i}) \Bigr\} \right] \nonumber \\ 
    = & E \left[ \frac{n}{n_k} \widehat{\zeta}_{k}(V_{i}) E \Bigl\{ \frac{I(A_i=a, R_i=k)}{\widehat{\pi}_{a, k}(X_{i})}  - I(R_i=k) \mid X_i \Bigr\}  \Bigl\{ m_{a, k}(X_i) -\widehat{m}_{a, k}(X_{i}) \Bigr\} \right] \nonumber \\ 
    = & E \left[ \frac{n}{n_k} \widehat{\zeta}_{k}(V_{i}) P(R_i=k \mid X_i) \Bigl\{ \frac{{\pi}_{a, k}(X_{i})}{\widehat{\pi}_{a, k}(X_{i})}  - 1 \Bigr\}  \Bigl\{ m_{a, k}(X_i) -\widehat{m}_{a, k}(X_{i}) \Bigr\} \right] \nonumber \\
    = & E \left[ \frac{\widehat{\zeta}_{k}(V_{i})}{\widehat{\pi}_{a, k}(X_{i})} \Bigl\{ {\pi}_{a, k}(X_{i})  - \widehat{\pi}_{a, k}(X_{i}) \Bigr\}  \Bigl\{ m_{a, k}(X_i) -\widehat{m}_{a, k}(X_{i}) \Bigr\} \right]
\end{align}
Given the boundedness of $\widehat{\zeta}_{k}$ and $\widehat{\pi}_{a, k}$, it can be concluded that $T_1 = O_p(\|\widehat{\pi}_{a, k}-\pi_{a, k} \| \|\widehat{m}_{a, k}-{m}_{a, k} \|)$. Now let us consider $T_2$. Conditioning on $V_i$, we have the following expression:
\begin{align}
    T_2  
    & = E \left[ \Bigl\{ \frac{n}{n_T} I(R_i = T) - \frac{n}{n_k} I(R_i = k) \widehat{\zeta}_{k}(V_{i}) \Bigr\} \Bigl\{ \widehat{\tau}_{a, k}(V_{i}) - \tau_{a, T}(V_{i}) \Bigr\} \right] \nonumber \\ 
    & = E \left[ E \Bigl\{ \frac{n}{n_T} I(R_i = T) - \frac{n}{n_k} I(R_i = k) \widehat{\zeta}_{k}(V_{i}) \mid V_i \Bigr\} \Bigl\{ \widehat{\tau}_{a, k}(V_{i}) - \tau_{a, T}(V_{i}) \Bigr\} \right] \nonumber \\
    & = E \left[ \Bigl\{ \frac{n}{n_T} P(R_i = T \mid V_i) - \frac{n}{n_k} P(R_i = k \mid V_i) \widehat{\zeta}_{k}(V_{i}) \Bigr\} \Bigl\{ \widehat{\tau}_{a, k}(V_{i}) - \tau_{a, T}(V_{i}) \Bigr\} \right] \nonumber \\
    & = E \left[ \Bigl\{ \frac{1}{P(R_i = T)} P(R_i = T \mid V_i) - \frac{1}{P(R_i = k)} P(R_i = k \mid V_i) \widehat{\zeta}_{k}(V_{i}) \Bigr\} \Bigl\{ \widehat{\tau}_{a, k}(V_{i}) - \tau_{a, T}(V_{i}) \Bigr\} \right] \nonumber \\
    & = E \left[  \frac{P(  R_i = k \mid V_i)}{P(R_i = k)}  \Bigl\{ {\zeta}_{k}(V_{i}) - \widehat{\zeta}_{k}(V_{i}) \Bigr\} \Bigl\{ \widehat{\tau}_{a, k}(V_{i}) - \tau_{a, T}(V_{i}) \Bigr\} \right]. 
\end{align}
Therefore, $T_2 = O_p(\|\widehat{\zeta}_k -\zeta_k \|\left\|\widehat{\tau}_{a, k}-\tau_{a, k} \right\|)$. 

Combining all results from the above three steps yields
\begin{align}
    \| \widehat{\mu}_{a, k}-\mu_{a, T} \| = O_p\left(n^{-1 / 2} + \|\widehat{\pi}_{a, k}-\pi_{a, k} \| \|\widehat{m}_{a, k}-{m}_{a, k} \|
    + \|\widehat{\zeta}_{k}-{\zeta}_{k} \| \|\widehat{\tau}_{a, k}-\tau_{a, k} \|
    \right). 
\end{align}

Further, if the nuisance estimators satisfy the following convergence rate
\begin{align}
    \left\|\widehat{m}_{a, k}-m_{a, k} \right\|\left\|\widehat{\pi}_{a, k}-\pi_{a, k} \right\| &= o_p(1 / \sqrt{n}), \quad
    \|\widehat{\zeta}_k -\zeta_k \|\left\|\widehat{\tau}_{a, k}-\tau_{a, k} \right\| =o_p(1 / \sqrt{n}), 
\end{align}
then the conditional bias, $E \left( \widehat{\mu}_{a, k}-\mu_{a, T}  \right) = o_p(1 / \sqrt{n})$ since $O_p(o_p(1 / \sqrt{n})) = o_p(1 / \sqrt{n})$. 

Consequently, by Slutsky's theorem, $\sqrt{n} ( \widehat{\mu}_{a,k} - \mu_{a, T} ) \stackrel{d}{\rightarrow} \mathcal{N}(0, \sigma^2)$. Here, $\sigma^2$ represents the semiparametric efficiency bound, which has been derived and defined in \eqref{equation_semiparametric_eff_bound}.
\end{proof}

\subsection{Proof of Theorem \ref{theorem_global_consistency_normality}}

\begin{proof}
    The proof presented herein closely follows the methodology outlined in \cite{han2021federated}; we start by establishing the consistency and asymptotic normality of the global estimator, assuming a fixed ${\eta}_k$. We then invoke Lemma 4 and 5 in \cite{han2021federated}, which state that the proposed adaptive estimation for $\eta_k$ as shown in \eqref{equation_l1_weights} allows for (i) the recovery of the optimal $\bar{\eta}_k$ by the estimator $\widehat{\eta}_k$, and (ii) the uncertainty introduced by $\widehat{\eta}_k$ is negligible when estimating $\Delta_T$.

    Lemma \ref{lemma_site_consistency} demonstrates the consistency of the site-specific estimators given that the source site $k$ satisfies the conditions outlined in Theorem \ref{theorem_site_consistency_normality}. We denote the set of source sites that fulfill the conditions in Theorem \ref{theorem_site_consistency_normality} as $\mathcal{S}^*$, and consider a fixed ${\eta}_k$ such that ${\eta}_k = 0$ for $k \not \in \mathcal{S}^*$, then 
    \begin{align}
        \widehat{\mu}_{a, G} (\eta_k) = \widehat{\mu}_{a, T} + \sum_{k \in \mathcal{K}} {\eta}_k \{\widehat{\mu}_{a, k} - \widehat{\mu}_{a, T} \} 
    \end{align}
    is consistent for $\mu_{a, T}$ since $\widehat{\mu}_{a, k}$ for $k \in \mathcal{S}^*$ are consistent estimators for ${\mu}_{a, T}$. Hence, we can establish that $\widehat{\Delta}_G (\eta_k) = \widehat{\mu}_{1, G} (\eta_k) - \widehat{\mu}_{0, G} (\eta_k)$ consistently estimates $\Delta_T = \mu_{1, T} - \mu_{0, T}$. Moving forward, we proceed to examine the asymptotic normality of the global estimator by utilizing the influence functions for the site-specific estimators. First, we rewrite the global estimator with the fixed $\eta_k$ as 
    \begin{align}
        \widehat{\Delta}_{G} (\eta_k)
        = \left(1 - \sum_{k \in \mathcal{S}} {\eta}_k \right) \left(  \widehat{\mu}_{1, T} - \widehat{\mu}_{0, T}  \right)+ \sum_{k \in \mathcal{S}} {\eta}_k \left(  \widehat{\mu}_{1, k} - \widehat{\mu}_{0, k}  \right) = \left(1 - \sum_{k \in \mathcal{S}} {\eta}_k \right) \widehat{\Delta}_T + \sum_{k \in \mathcal{S}} {\eta}_k \widehat{\Delta}_k
    \end{align}
    In Appendix \ref{appendix_site_estimator_influence}, the influence functions for the source site and target site estimators have been derived. To facilitate representation, we decompose the influence functions for the source site estimators into two parts, each defined on the target sample and source sample, respectively,
    \begin{equation}
    \begin{aligned}[b]
    \xi^{(1)}_{a, k}(Z_i) & =\frac{1}{P(R_i=k)} \Biggl[ \frac{I(A_i=a, R_i=k)}{ \pi_{a, k}(X_i)} \zeta_{k}(V_i) \{ Y_i -m_{a, k}(X_i) \} \Biggr] \\
    & + \frac{1}{P(R_i =k)} 
    \Biggl[ I(R_i=k) \zeta_{k}(V_i) \{ m_{a,k}(X_i)-\tau_{a, k}(V_i) \} \Biggr] \\
    \xi^{(2)}_{a, k}(Z_i) & = \frac{1}{P(R_i=T)} \Biggl[ I(R_i=T)  \tau_{a, k}(V_i) \Biggr] - \mu_{a, T}. 
    \end{aligned}
    \end{equation}  
    Here, it should be noted that the expectations of 
    where $\xi^{(1)}_{a, k}(Z_i)$ and $\xi^{(2)}_{a, k}(Z_i)$ are both equal to zero. Furthermore, for the target site estimator, we define:
    \begin{equation}
    \begin{aligned}[b]
    \xi^{(2)}_{a, T}(Z_i) & = \frac{1}{P(R_i=T)} \Biggl[ \frac{I(A_i=a, R_i=T)}{ \pi_{a, T}(V_i)} \{ Y_i-m_{a, T}(V_i) \} + I(R_i=T)  m_{a, T}(V_i)\Biggr] -\mu_{a, T}
    \end{aligned}
    \end{equation}
    and the expectation of $\xi^{(2)}_{a, T}(Z_i)$ is also zero. We denote the influence function for $\widehat{\Delta}_{k}$ as $\xi^{(1)}_{1, k}(Z_i) - \xi^{(1)}_{0, k}(Z_i) + \xi^{(2)}_{1, k}(Z_i) - \xi^{(2)}_{0, k}(Z_i)$. 
    Similarly, the influence function for $\widehat{\Delta}_{T}$ is given by $\xi^{(2)}_{1, T}(Z_i) - \xi^{(2)}_{0, T}(Z_i)$. Thus, the influence function for the global estimator evaluated with target and source samples can be expressed as:
    {\allowdisplaybreaks
    \begin{align}
        \widehat{\Delta}_{G} (\eta_k) - \Delta_{T}
        & = \left(1 - \sum_{k \in \mathcal{S}} {\eta}_k \right) \left(\widehat{\Delta}_{T} - \Delta_{T} \right) + \sum_{k \in \mathcal{S}} {\eta}_k \left(\widehat{\Delta}_{k}- \Delta_{T} \right) \nonumber \\
        & = \left(1 - \sum_{k \in \mathcal{S}} {\eta}_k \right) 
        \frac{1}{n_T} \sum_{i = 1}^N I(R_i = T) \left\{\xi^{(2)}_{1, T}(Z_i) - \xi^{(2)}_{0, T}(Z_i)\right\} \nonumber \\
        & + \left( \sum_{k \in \mathcal{S}} {\eta}_k \right)\frac{1}{n_T}\sum_{i = 1}^N I(R_i = T) \left\{\xi^{(2)}_{1, k}(Z_i) - \xi^{(2)}_{0, k}(Z_i)\right\} \nonumber \\
        & + \sum_{k \in \mathcal{S}} {\eta}_k \frac{1}{n_k}\sum_{i = 1}^N I(R_i = k) \left\{\xi^{(1)}_{1, k}(Z_i) - \xi^{(1)}_{0, k}(Z_i)\right\} \nonumber \\
        & = 
        \frac{1}{N} \sum_{i = 1}^N I(R_i = T) \left(1 - \sum_{k \in \mathcal{S}} {\eta}_k \right)  \frac{\left\{\xi^{(2)}_{1, T}(Z_i) - \xi^{(2)}_{0, T}(Z_i)\right\}}{P(R_i = T)} \nonumber \\
        & + \frac{1}{N} \sum_{i = 1}^N I(R_i = T) \left(\sum_{k \in \mathcal{S}} {\eta}_k \right) \frac{\left\{\xi^{(2)}_{1, k}(Z_i) - \xi^{(2)}_{0, k}(Z_i)\right\}}{P(R_i = T)} \nonumber \\
        & + \frac{1}{N} \sum_{k \in \mathcal{S}} \sum_{i = 1}^N I(R_i = k) {\eta}_k \frac{\left\{\xi^{(1)}_{1, k}(Z_i) - \xi^{(1)}_{0, k}(Z_i)\right\}}{P(R_i = k)}  
        \label{equation_influence_function_global}
    \end{align}}
    
    The asymptotic variance for $\widehat{\Delta}_{G} (\eta_k)$ equals the variance of the influence function \eqref{equation_influence_function_global}. Its derivation is similar to that of \ref{equation_semiparametric_eff_bound}. Let us denote this asymptotic variance as ${\mathcal{V}} (\eta_k)$. Under the assumption of i.i.d. units within each site, we have: 
    \begin{align}
        {\mathcal{V}} (\eta_k)
        & = \left(1 - \sum_{k \in \mathcal{S}} {\eta}_k \right)^2  \frac{\text{Var} \left\{\xi^{(2)}_{1, T}(Z_i) - \xi^{(2)}_{0, T}(Z_i) \mid R_i = T\right\}}{P(R_i = T)} \nonumber \\
        & + \left(\sum_{k \in \mathcal{S}} {\eta}_k\right)^2  \frac{\text{Var} \left\{\xi^{(2)}_{1, k}(Z_i) - \xi^{(2)}_{0, k}(Z_i) \mid R_i = T \right\}}{P(R_i = T)} \nonumber \\
        & + 2 \left(1 - \sum_{k \in \mathcal{S}} {\eta}_k \right) \left(\sum_{k \in \mathcal{S}} {\eta}_k \right)  \frac{\text{Cov} \left\{\xi^{(2)}_{1, T}(Z_i) - \xi^{(2)}_{0, T}(Z_i), \xi^{(2)}_{1, k}(Z_i) - \xi^{(2)}_{0, k}(Z_i) \mid R_i = T \right\}}{P(R_i = T)} \nonumber \\
        & + \sum_{k \in \mathcal{S}} {\eta}_k^2 \frac{\text{Var}\left\{\xi^{(1)}_{1, k}(Z_i) - \xi^{(1)}_{0, k}(Z_i) \mid R_i = k \right\}}{P(R_i = k)}
    \end{align}
    Under the uniform boundedness conditions stated in Theorem \ref{theorem_site_consistency_normality}, this variance is finite. Consequently, we can express the asymptotic distribution of $\sqrt{N}\left(\widehat{\Delta}{G} (\eta_k) -\Delta{T}\right)$ as:
    \begin{equation}
        \sqrt{N}\left(\widehat{\Delta}_{G} (\eta_k) -\Delta_{T}\right) \stackrel{d}{\rightarrow} \mathcal{N}\left(0, {\mathcal{V}} (\eta_k) \right)
    \end{equation}
    We further define the optimal adaptive weight $\bar{\eta}_k$ as follows:
    \begin{equation}
        \bar{\eta}_k = \arg \min_{{\eta}_k = 0 \forall k \not \in \mathcal{S}^*} \mathcal{V} (\eta_k)
    \end{equation}
    
    By leveraging Lemmas 4 and 5 from \cite{han2021federated}, we can recover the optimal $\bar{\eta}_k$ with negligible uncertainty for estimating $\Delta_T$ if we estimate $\widehat{\eta}_{k, L_1}$ using \eqref{equation_l1_weights}. The consistency of $\widehat{\mathcal{V}}(\widehat{\eta}_{k, L_1})$ follows when we can effectively approximate $\mathcal{V}(\bar{\eta}_k)$ with $\widehat{\mathcal{V}}(\widehat{\eta}_{k, L_1})$. Thus,
    \begin{equation}
        \sqrt{N / \widehat{\mathcal{V}} (\widehat{\eta}_{k, L_1}) }\left(\widehat{\Delta}_{G}(\widehat{\eta}_{k, L_1}) -\Delta_{T}\right) \stackrel{d}{\rightarrow} \mathcal{N}\left(0, 1\right)
    \end{equation}

    We now proceed to analyze the efficiency gain resulting from the federation process. The estimator that relies solely on the target data is denoted as $\widehat{\Delta}_T = \widehat{\Delta}_G(\eta_0)$, where $\eta_0$ assigns all weights to the target and none to the source.
    In contrast, the estimator that leverages the proposed adaptive ensemble approach is denoted as $\widehat{\Delta}_G(\widehat{\eta}_{k, L_1}) = (1 - \sum_{k \in \mathcal{S}} \widehat{\eta}_{k, L_1} ) \widehat{\Delta}_T + \sum_{k \in \mathcal{S}} \widehat{\eta}_{k, L_1} \widehat{\Delta}_k$. Here, $\widehat{\eta}_{k, L_1}$ can recover the optimal weight  $\bar{\eta}_{k}$ that is associated with the minimum asymptotic variance. Consequently, the variance of $\widehat{\Delta}_G(\widehat{\eta}_{k, L_1})$ is no larger than that of the estimator relying solely on the target data by definition.

    To establish that the asymptotic variance of $\widehat{\Delta}_G(\widehat{\eta}_{k, L_1})$ is strictly smaller than that of the estimator based solely on the target data $\widehat{\Delta}_T$, we adapt Proposition 1 in \cite{han2021federated} with a modified informative source condition. Specifically, for each source site $s \in \mathcal{S}^*$, we define $\widehat{\Delta}_{G}(\eta_{s})$ a global estimator where $\eta_{s}$ is the optimal ensemble weight if we only consider target site and this source site $s$. Then the modified informative source condition is $\operatorname{Cov}\left\{ \sqrt{N} \widehat{\Delta}_{T}, \sqrt{N}\left(\widehat{\Delta}_{G}({\eta}_{s}) -\widehat{\Delta}_T\right) \right\}$, where $\widehat{\Delta}_{G}({\eta}_{s}) -\widehat{\Delta}_T$ can be expressed as
    \begin{align}
        \widehat{\Delta}_{G}({\eta}_{s}) -\widehat{\Delta}_T
        & = \widehat{\Delta}_{G}({\eta}_{s}) - \Delta_T - \left( \widehat{\Delta}_T - \Delta_T \right) \nonumber \\ 
        & = 
        \frac{1}{N} \sum_{i = 1}^N I(R_i = T) \left(1 - {\eta}_{s} \right)  \frac{\left\{\xi^{(2)}_{1, T}(Z_i) - \xi^{(2)}_{0, T}(Z_i)\right\}}{P(R_i = T)} \nonumber \\
        & + \frac{1}{N} \sum_{i = 1}^N I(R_i = T) {\eta}_{s} \frac{\left\{\xi^{(2)}_{1, s}(Z_i) - \xi^{(2)}_{0, s}(Z_i)\right\}}{P(R_i = T)} \nonumber \\
        & + \frac{1}{N} \sum_{i = 1}^N I(R_i = s) {\eta}_{s} \frac{\left\{\xi^{(1)}_{1, s}(Z_i) - \xi^{(1)}_{0, s}(Z_i)\right\}}{P(R_i = s)}\nonumber \\
        & - \frac{1}{N} \sum_{i = 1}^N I(R_i = T) \frac{\left\{\xi^{(2)}_{1, T}(Z_i) - \xi^{(2)}_{0, T}(Z_i)\right\}}{P(R_i = T)} \nonumber \\
        & = \frac{1}{N} \sum_{i = 1}^N I(R_i = T) {\eta}_{s} \frac{\left\{\xi^{(2)}_{1, s}(Z_i) - \xi^{(2)}_{0, s}(Z_i) - \xi^{(2)}_{1, T}(Z_i) + \xi^{(2)}_{0, T}(Z_i)
        \right\}}{P(R_i = T)} \nonumber \\ 
        & + \frac{1}{N} \sum_{i = 1}^N I(R_i = s) {\eta}_{s} \frac{\left\{\xi^{(1)}_{1, s}(Z_i) - \xi^{(1)}_{0, s}(Z_i)\right\}}{P(R_i = s)}. 
    \end{align}
    In summary, if there exist a source site $s$ with consistent estimator of $\Delta_{T}$ and further satisfy $\left|\operatorname{Cov}\left\{ \sqrt{N} \widehat{\Delta}_{T}, \sqrt{N}\left(\widehat{\Delta}_{G} ( {\eta}_{s}) -\widehat{\Delta}_T\right) \right\} \right| \geq \varepsilon$ where $\varepsilon$ denotes a positive constant, the variance of $\widehat{\Delta}_{G}$ is strictly smaller than $\widehat{\Delta}_{T}$. 
\end{proof}

% References follow the acknowledgments in the camera-ready paper. Use unnumbered first-level heading for the references. Any choice of citation style is acceptable as long as you are consistent. It is permissible to reduce the font size to \verb+small+ (9 point) when listing the references. Note that the Reference section does not count towards the page limit.
%\medskip

%{
%\small

%[1] Alexander, J.A.\ \& Mozer, M.C.\ (1995) Template-based algorithms for connectionist rule extraction. In G.\ Tesauro, D.S.\ Touretzky and T.K.\ Leen (eds.), {\it Advances in Neural Information Processing Systems 7}, pp.\ 609--616. Cambridge, MA: MIT Press.
%[2] Bower, J.M.\ \& Beeman, D.\ (1995) {\it The Book of GENESIS: Exploring Realistic Neural Models with the GEneral NEural SImulation System.}  New York: TELOS/Springer--Verlag.

%[3] Hasselmo, M.E., Schnell, E.\ \& Barkai, E.\ (1995) Dynamics of learning and recall at excitatory recurrent synapses and cholinergic modulation in rat hippocampal region CA3. {\it Journal of Neuroscience} {\bf 15}(7):5249-5262.

%}

%%%%%%%%%%%%%%%%%%%%%%%%%%%%%%%%%%%%%%%%%%%%%%%%%%%%%%%%%%%%

\end{document}